\def\year{2022}\relax
\documentclass[letterpaper]{article} 
\usepackage{aaai22}  
\usepackage{times}  
\usepackage{helvet}  
\usepackage{courier}  
\usepackage[hyphens]{url}  
\usepackage{graphicx} 
\usepackage{natbib}  
\usepackage{caption} 
\DeclareCaptionStyle{ruled}{labelfont=normalfont,labelsep=colon,strut=off} 
\usepackage{fancyhdr}
\usepackage{algorithm}
\usepackage{algorithmic}

\usepackage{dsfont}
\usepackage{mathtools}
\usepackage{amsmath,amssymb}
\DeclareMathOperator{\E}{\mathbb{E}}
\DeclareMathOperator{\Var}{\mathbb{V}}
\let\P\relax
\newcommand{\bigO}{\mathcal{O}}
\DeclareMathOperator{\P}{\mathbb{P}}
\DeclareMathOperator{\defeq}{\dot{=}}

\DeclareMathOperator{\ind}{\mathds{1}}
\DeclareMathOperator{\pa}{pa}
\DeclareMathOperator{\ch}{ch}

\DeclareMathOperator{\an}{an}
\DeclareMathOperator{\mb}{mb}

\usepackage{xcolor}         

\usepackage{verbatim}

\newcommand\numberthis{\addtocounter{equation}{1}\tag{\theequation}}

\usepackage{amsthm}
\newtheorem{theorem}{Theorem}

\usepackage{newfloat}
\usepackage{listings}
\floatstyle{ruled}
\newfloat{listing}{tb}{lst}{}
\floatname{listing}{Listing}
\title{Hindsight Network Credit Assignment:\\Efficient Credit Assignment in Networks of Discrete Stochastic Units}
\author {
    Kenny Young
}
\affiliations{
Department of Computing Science\\
University of Alberta, Amii\\
Edmonton, Canada\\
 kjyoung@ualberta.ca
}

\begin{document}

\maketitle


{\let\thefootnote\relax\footnote{{Code to reproduce the experiments is available at: https://github.com/kenjyoung/HNCA\_code\_supplement.}}}
\addtocounter{footnote}{-1}

\begin{abstract}
Training neural networks with discrete stochastic variables presents a unique challenge. Backpropagation is not directly applicable, nor are the reparameterization tricks used in networks with continuous stochastic variables. To address this challenge, we present Hindsight Network Credit Assignment (HNCA), a novel gradient estimation algorithm for networks of discrete stochastic units. HNCA works by assigning credit to each unit based on the degree to which its output influences its immediate children in the network. We prove that HNCA produces unbiased gradient estimates with reduced variance compared to the REINFORCE estimator, while the computational cost is similar to that of backpropagation. We first apply HNCA in a contextual bandit setting to optimize a reward function that is unknown to the agent. In this setting, we empirically demonstrate that HNCA significantly outperforms REINFORCE, indicating that the variance reduction implied by our theoretical analysis is significant and impactful. We then show how HNCA can be extended to optimize a more general function of the outputs of a network of stochastic units, where the function is known to the agent. We apply this extended version of HNCA to train a discrete variational auto-encoder and empirically show it compares favourably to other strong methods. We believe that the ideas underlying HNCA can help stimulate new ways of thinking about efficient credit assignment in stochastic compute graphs.
\end{abstract}

Using discrete stochastic units within neural networks is appealing for a number of reasons, including representing multimodal distributions, modeling discrete choices, providing regularization and facilitating exploration. However, training such units efficiently and accurately presents challenges, as backpropagation is not directly applicable, nor are the reparameterization tricks~\citep{kingma2013auto, rezende2014stochastic} that are typically used with continuous stochastic units. Despite these challenges, discrete stochastic units have played an important role in recent empirical successes in both text-to-image generation
~\citep{ramesh2021zero} and model based reinforcement learning~\citep{hafner2021mastering}. Hence, techniques for efficiently training networks of discrete stochastic units have the potential to be of significant practical interest.  

Prior work has proposed a number of techniques for producing either biased, or unbiased estimates of gradients for discrete stochastic units. \citet{bengio2013estimating} propose an unbiased REINFORCE~\citep{williams1992simple} style estimator, as well as a biased but low variance estimator which replaces a random variable with its expectation during backpropagation. \citet{tang2013learning} propose an EM procedure which maximizes a variational lower bound on the loss. \citet{mnih2014neural} propose several techniques to reduce the variance of a REINFORCE style estimator, including subtracting a learned baseline and normalizing by a moving average standard deviation. \citet{maddison2016concrete} and \citet{jang2016categorical} each propose a biased estimator based on a continuous relaxation of discrete outputs. \citet{tucker2017rebar} use such a continuous relaxation to derive a control variate for a REINFORCE style estimator, resulting in a variance reduced \textit{unbiased} gradient estimator. \citet{grathwohl2017backpropagation} and \citet{gu2015muprop} also explore the use of control variates with discrete random variables. \citet{yin2019arm} provide a variance reduced unbiased estimator, called ARM, based on a particular reparameterization and antithetic sampling. \citet{dong2020disarm} further reduce the variance of ARM by marginalizing over the reparameterization step.

We introduce an unbiased, and computationally efficient estimator for the gradients of stochastic units which provably reduces gradient estimate variance compared to REINFORCE. Our estimator works by assigning credit to each unit based on how much it impacts the outputs of its immediate children. Our approach is inspired by Hindsight Credit Assignment (HCA; \citet{harutyunyan2019hindsight}) for reinforcement learning (RL), hence we call it Hindsight Network Credit Assignment (HNCA).

Aside from HCA, perhaps the most closely related work is the Local Expectation Gradients (LEG) approach of \citet{titsias2015local}. In fact, the gradient estimator used in HNCA can be seen as an instance of the LEG estimator. However, the generic expression for the LEG estimator makes it unclear when and how it can be efficiently computed. This has led to suggestions in the literature that LEG tends to be too computationally expensive to be practical~\citep{tucker2017rebar,mnih16vimco}. 

The present work extends the work of \citet{titsias2015local} in several ways. First, while LEG may be computational expensive in the general case, for the common case of a network of Bernoulli units, with firing probability parameterized by a linear transformation of their inputs followed by a nonlinear activation, HNCA provides an efficient message passing procedure.\footnote{A similarly procedure applies to units with softmax activation, though we do not explore this empirically in this work.} In this case, the resulting computational cost is similar to that of Backpropagation. This efficiency allows us to straightforwardly apply HNCA to multi-layer Bernoulli networks, while the analysis and experiments of \citet{titsias2015local} focus on single-layer (fully factorized) stochastic networks. We further demonstrate that a simple baseline subtraction, similar to that employed by~\citet{mnih2014neural}, drastically improves performance when applying HNCA to multi-layer networks. While \citet{titsias2015local} focus on the case where the agent has access to the function being optimized, we also present HNCA in a contextual bandit setting where an agent operates online, outputting an action at each time-step and observing a single sampled reward as a result. Interestingly, in the contextual bandit setting, we can still compute local expectations for each hidden unit without the need to resample the reward. Finally, we prove that HNCA provides a variance reduction over REINFORCE.

In taking inspiration from RL to train networks of stochastic units, our work is related to work on CoAgent Networks~\citep{ thomas2011conjugate,kostas2020asynchronous} that formalizes framing stochastic networks as collectives of interacting RL agents.


In addition to the immediate application to stochastic neural networks, we believe the insights presented in this work can help pave the way for new ways of thinking about efficient credit assignment in stochastic compute graphs, including perhaps the RL setting.

\section{HNCA in a Contextual Bandit Setting}
\label{bandit}
We first formulate HNCA in a contextual bandit setting. In this setting, an agent interacts with an environment in a series of time-steps.\footnote{We suppress the time-step, for example writing the context as $X$ instead of $X_t$.} At each time-step, the environment presents the agent with an i.i.d. random context $X\in\mathcal{X}$ (for example the pixels of an image). The agent then selects an action from a discrete set of choices $A\in\mathcal{A}$ (for example a guess of what class the image belongs to). The environment responds to the agent's choice with a reward $R=R(X,A)$, where $R: \mathcal{X},\mathcal{A}\mapsto\mathbb{R}$ is an unknown reward function (for example a reward of 1 for guessing the correct class and 0 otherwise). The agent's goal is to select actions which result in as much reward as possible.

In our case, the agent consists of a network of stochastic computational units. Let $\Phi$ be a random variable corresponding to the output of a particular unit. For each unit, $\Phi$ is drawn from a parameterized \textit{policy} ${\pi_\Phi(\phi|b)\defeq\P(\Phi=\phi|\pa(\Phi)=b)}$ conditioned on $\pa(\Phi)=b$, its parents in the network.\footnote{Expectations and probabilities are taken with respect to all random variables in the network, and the context.} Each unit's policy is differentiably parameterized by a unique set of parameters $\theta_\Phi\in\mathbb{R}^d$. A unit's parents $\pa(\Phi)$ may include the output of other units, as well as the context $X$. We focus on the case where $\Phi$ takes values from a discrete set. We will use $\ch(\Phi)$ to refer to the children of $\Phi$, that is, the set of outputs of all units for which $\Phi$ is an input.\footnote{We may also apply $\ch(\cdot)$ or $\pa(\cdot)$ to sets, in which case it has the obvious meaning of the union of the elementwise applications.} We assume the network has a single output unit, which selects the action $A$ sent to the environment.

The goal is to tune the network parameters to increase $\E[R]$. Towards this, we will construct an unbiased estimator of the gradient $\frac{\partial\E[R]}{\partial\theta_\Phi}$ for the parameters of each unit, and update the parameters according to the estimator. 

Directly computing the gradient of the output probability with respect to the parameters for a given input, as we might do with backpropagation for a deterministic network, is generally intractable for discrete stochastic networks. Instead, we can define a local REINFORCE estimator, $\hat{G}^{\text{RE}}_{\Phi}\defeq\frac{\partial\log(\pi_\Phi(\Phi|\pa(\Phi)))}{\partial\theta_\Phi}R$. It is well known that $\hat{G}^{\text{RE}}_{\Phi}$ is an unbiased estimator of $\frac{\partial\E[R]}{\partial\theta_\Phi}$ (see Appendix~\ref{local_REIN_unbiased} for a proof). However, $\hat{G}^{\text{RE}}_{\Phi}$ tends to have high variance.


\subsection{HNCA Gradient Estimator}
HNCA exploits the causal structure of the network to assign credit to each unit's output based on how it impacts the output of its immediate children. Assume $\Phi$ is a nonoutput unit and define $\mb(\Phi)\defeq\{\ch(\Phi),\pa(\Phi),\pa(\ch(\Phi))\setminus\Phi\}$ as a notational shorthand. Note that $\mb(\Phi)$ is a Markov blanket~\citep{pearl1988probabilistic} for $\Phi$, meaning that conditioned on $\mb(\Phi)$, $\Phi$ is independent of all other variables in the network as well as the reward $R$. Beginning from the expression for $\hat{G}^{\text{RE}}_{\Phi}$, we can rewrite $\frac{\partial\E[R]}{\partial\theta_\Phi}$ as follows:
\begin{align*}
    \frac{\partial\E[R]}{\partial\theta_\Phi}
    &\stackrel{(a)}{=}\E\left[\frac{\partial\log(\pi_\Phi(\Phi|\pa(\Phi)))}{\partial\theta_\Phi}R\right]\\
    &\stackrel{(b)}{=}\E\left[\E\left[\frac{\partial\log(\pi_\Phi(\Phi|\pa(\Phi)))}{\partial\theta_\Phi}R\middle|\mb(\Phi),R\right]\right]\\
    &\stackrel{(c)}{=}\E\left[\E\left[\frac{\partial\log(\pi_\Phi(\Phi|\pa(\Phi)))}{\partial\theta_\Phi}\middle|\mb(\Phi)\right]R\right]\\
    &\stackrel{(d)}{=}\E\left[\sum_\phi\frac{\P(\Phi=\phi|\mb(\Phi))}{\pi_{\Phi}(\phi|\pa(\Phi))}\frac{\partial\pi_\Phi(\phi|\pa(\Phi))}{\partial\theta_\Phi}R\right],\numberthis\label{HNCA_derivation_1}
\end{align*}
where $(a)$ follows from the unbiasedness of $\hat{G}^{\text{RE}}$, $(b)$ applies the law of total expectation, $(c)$ pulls $R$ out of the expectation and then uses the fact that $\mb(\Phi)$ forms a Markov blanket for $\Phi$, thus we can drop the conditioning on $R$ without loosing anything, and $(d)$ expands the inner expectation over $\Phi$ and rewrites the log gradient. This idea of taking a local expectation conditioned on a Markov blanket is similar to the LEG estimator proposed by~\citet{titsias2015local}. However, it is not immediately obvious how to compute this estimator efficiently. \citet{titsias2015local} provide a more explicit expression and empirical results for a fully factorized variational distribution. Here, we will go beyond this case to provide a computationally efficient way to compute the inner expression for more general networks of stochastic units. To begin, we apply Theorem 1 from Chapter 4 of the probabilistic reasoning textbook of~\citet{pearl1988probabilistic}, which implies that
\begin{equation}\label{markov_blanket_conditional_expression}
    \P(\Phi=\phi|\mb(\Phi))=\rho_\Phi(\phi)\pi_\Phi(\phi|\pa(\Phi)).
\end{equation}
where $\rho_\Phi(\phi)=\frac{\prod\limits_{C\in\ch(\Phi)}\pi_C(C|\pa(C)\setminus\Phi,\Phi=\phi)}{\sum\limits_{\phi'}\pi_\Phi(\phi'|\pa(\Phi))\prod\limits_{C\in\ch(\Phi)}\pi_C(C|\pa(C)\setminus\Phi,\Phi=\phi')}$. Intuitively,  $\rho_\Phi(\phi)$ is the relative counterfactual probability of the children of $\Phi$ taking the value they did had $\Phi$ been fixed to $\phi$. See Appendix~\ref{markov_blanket_conditional} for a full proof. Substituting this result into the expression within the expectation in Equation~\ref{HNCA_derivation_1}, we get that the following is an unbiased estimator of $\frac{\partial\E[R]}{\partial\theta_\Phi}$:
\begin{equation}\label{HNCA_estimator}
    \hat{G}^{\text{HNCA}}_{\Phi}\defeq\sum_\phi\rho_\Phi(\phi)\frac{\partial\pi_\Phi(\phi|\pa(\Phi))}{\partial\theta_\Phi}R,
\end{equation}
which we call the HNCA gradient estimator. Equation~\ref{HNCA_estimator} applies only to units for which $\ch(\Phi)\neq\emptyset$ and thus excludes the output unit $A$. In our contextual bandit experiments, we use the REINFORCE estimator $\hat{G}^{\text{RE}}_\Phi(\phi)$ for the output unit, in Section~\ref{extended} we show how to improve upon this if we have access to the reward function.

HNCA assigns credit to a particular output choice $\phi$ based on the relative counterfactual probability of its children's outputs had $\phi$ been chosen, independent of the actual value of $\Phi$. Intuitively, this reduces variance, because each potential output choice of a given unit will get credit proportional to the difference it makes further downstream. On the other hand, REINFORCE credits whatever output happens to be selected, whether it makes a difference or not. This intuition is formalized in the following theorem: 
\begin{theorem}\label{reduced_variance}
$\Var(\hat{G}^{\text{HNCA}}_{\Phi})\leq \Var(\hat{G}^{\text{RE}}_{\Phi})$, where $\Var(\vec{X})$ stand for the elementwise variance of random vector $\vec{X}$, and the inequality holds elementwise.
\end{theorem}
Theorem~\ref{reduced_variance} follows from the law of total variance by the proof available in Appendix~\ref{HNCA_action_value_low_var}.

\subsection{Efficient Implementation of HNCA}
We implement HNCA as a message-passing procedure. A forward pass propagates information from parents to children to compute the network output. A backward pass passes information from children to parents to compute the HNCA gradient estimator. The computational complexity of this procedure depends on how difficult it is to compute the numerators of $\rho_\Phi(\phi)$. We could naively recompute $\pi_C(C|\pa(C)\setminus\Phi,\Phi=\phi)$ from scratch for each possible $\phi$. When $C$ corresponds to a Bernoulli unit, which computes its output probability as a linear function of its inputs followed by sigmoid activation, this would require time $\bigO(|\pa(C)|N_{\Phi})$, where $N_{\Phi}$ is the number of possible values $\Phi$ can take (2 if $\Phi$ is also Bernoulli). To do this for every parent of every unit in a Bernoulli network would thus require $\bigO(2\sum_\Phi|\pa(\Phi)|^2)$. This is much greater than the cost of a forward pass, which takes on the order of the total number of edges in the network, or $\bigO(\sum_\Phi|\pa(\Phi)|)$. This contrasts with backpropagation where the cost of the backward pass is on the same order as the forward pass, an appealing property, which implies that learning is not a bottleneck.
\begin{figure}[htb]
\begin{algorithm}[H]
\begin{algorithmic}[1]
\STATE Receive $\vec{x}$ from parents\\
\STATE $l=\vec{\theta}\cdot\vec{x}+b$\\
\STATE $p=\sigma(l)$\\
\STATE $\phi\sim \textit{Bernoulli}(p)$\\
\STATE Pass $\phi$ to children\\
\STATE Receive $\vec{q}_1,\vec{q}_0,R$ from children\\
\STATE $q_1=\prod_{i}\vec{q}_1[i]$\\
\STATE $q_0=\prod_{i}\vec{q}_0[i]$\\
\STATE $\bar{q}=pq_1+(1-p)q_o$\\
\STATE $\vec{l}_1=l+\vec{\theta}\odot(1-\vec{x})$\\
\STATE $\vec{l}_0=l-\vec{\theta}\odot\vec{x}$\\
\STATE $\vec{p}_1=(1-\phi)(1-\sigma(\vec{l}_1))+\phi\sigma(\vec{l}_1)$\\
\STATE $\vec{p}_0=(1-\phi)(1-\sigma(\vec{l}_0))+\phi\sigma(\vec{l}_0)$\\
\STATE Pass $\vec{p}_1,\vec{p}_0,R$ to parents\\
\STATE $\vec{\theta}=\vec{\theta}+\alpha\sigma^{\prime}(l)\vec{x}\left(\frac{q_1-q_0}{\bar{q}}\right)R$\\
\STATE $b=b+\alpha\sigma^{\prime}(l)\left(\frac{q_1-q_0}{\bar{q}}\right)R$
\end{algorithmic}
\caption{HNCA (Bernoulli unit)}\label{bernoulli_alg}
\end{algorithm}
\addtocounter{algorithm}{-1}
\captionof{algorithm}{The forward pass in lines 1-5 takes input from the parents, uses it to compute the fire probability $p$ and samples $\phi\in\{0,1\}$. The backward pass receives two vectors of probabilities $\vec{q}_1$ and $\vec{q}_0$, each with one element for each child. Each element represents $\vec{q}_{0/1}[i]=\P\left(C_i\middle|\pa(C_i)\setminus\Phi,\Phi=0/1\right)$ for a given child $C_i\in \ch(\Phi)$. Lines 7 and 8 take the product of child probabilities to compute $\prod_i\pi_{C_i}(C_i|\pa(C_i)\setminus\Phi,\Phi=0/1)$. Line 9 computes the associated normalizing factor. Line 10-13 use the logit $l$ to efficiently compute a vector of probabilities $\vec{p}_1$ and $\vec{p}_0$. Each element corresponds to a counterfactual probability of $\phi$ if a given parent's value was fixed to 1 or 0. Here $\odot$ represents the elementwise product. Line 14 passes information to the unit's children. Lines 15 and 16 finally update the parameter using $\hat{G}^{\text{HNCA}}_{\Phi}$ with learning-rate hyperparameter $\alpha$.}
\end{figure}

Luckily, we can improve on this for cases where $\pi_C(C|\pa(C)\setminus\Phi,\Phi=\phi)$ can be computed from $\pi_C(C|\pa(C))$ in less time than computing $\pi_C(C|\pa(C))$ from scratch. This is indeed the case for linear Bernoulli units, for which the policy can be written $\pi_\Phi(\phi|\vec{x})=\sigma(\vec{\theta}\cdot\vec{x}+b)$ where $\vec{x}$ is the binary vector consisting of all parent outputs, $b$ is a scalar bias, $\vec{\theta}$ is the parameter vector for the unit, and $\sigma$ is the sigmoid function. Say we wish to compute the counterfactual probability of $\Phi=1$ given $\vec{x}[i]=1$, if we already have $\pi_\Phi(1|\vec{x})$. Regardless of the actual value of $\vec{x}_i$ we can use the following identity:
\begin{equation*}
    \pi_\Phi(1|\vec{x}\setminus \vec{x}[i],\vec{x}[i]=1)=\sigma(\sigma^{-1}(\pi_\Phi(1|\vec{x}))+\vec{\theta}[i](1-\vec{x}[i])).
\end{equation*}
 
This requires only constant time, whereas computing $\pi_\Phi(\phi|\vec{x})$ requires time proportional to the length of $\vec{x}$. This simple idea is crucial for implementing HNCA efficiently. In this case, we can compute the numerator terms for every unit in a Bernoulli network in $\bigO(\sum_\Phi|\pa(\Phi)|)$ time. This is now on the same order as computing a forward pass through the network. Computing $\hat{G}_{\Phi}^{\text{HNCA}}$ for a given $\Phi$ from these numerator terms requires multiplying a scalar by a gradient vector with the same size as $\theta_\Phi$. For a Bernoulli unit, $\theta_\Phi$ has $\bigO(|\pa(\Phi)|)$ elements, so this operation adds another $\bigO(\sum_\Phi|\pa(\Phi)|)$, maintaining the same order of complexity.


Algorithm~\ref{bernoulli_alg} shows an efficient implementation of HNCA for Bernoulli units. Note that, for ease of illustration, the pseudocode is implemented for a single unit and a single training example at a time. In practice, we use a vectorized version which works with vectors of units that constitute a layer, and with minibatches of training data. 


In Section~\ref{bandit_experiments}, we will apply HNCA to a model consisting of a number of hidden layers of Bernoulli units followed by a softmax output layer. Appendix~\ref{softmax_alg} provides an implementation and discussion of HNCA for a softmax output unit. Note that the output unit itself uses the REINFORCE estimator in its update, as it has no children, which precludes the use of HNCA. Nonetheless, the output unit still needs to provide information to its parents, which do use HNCA. Using a softmax unit at the output, we can still maintain the property that the time required for the backward pass is on the same order as the time required for the forward pass. If, on the other hand, the entire network consisted of softmax nodes with $N$ choices each, the HNCA backward pass would require a factor of $N$ more computation than the forward pass, we discuss this in Appendix~\ref{softmax_alg} as well.

\subsection{Contextual Bandit Experiments}
\label{bandit_experiments}
\begin{figure}[htb]
    \includegraphics[width=\columnwidth]{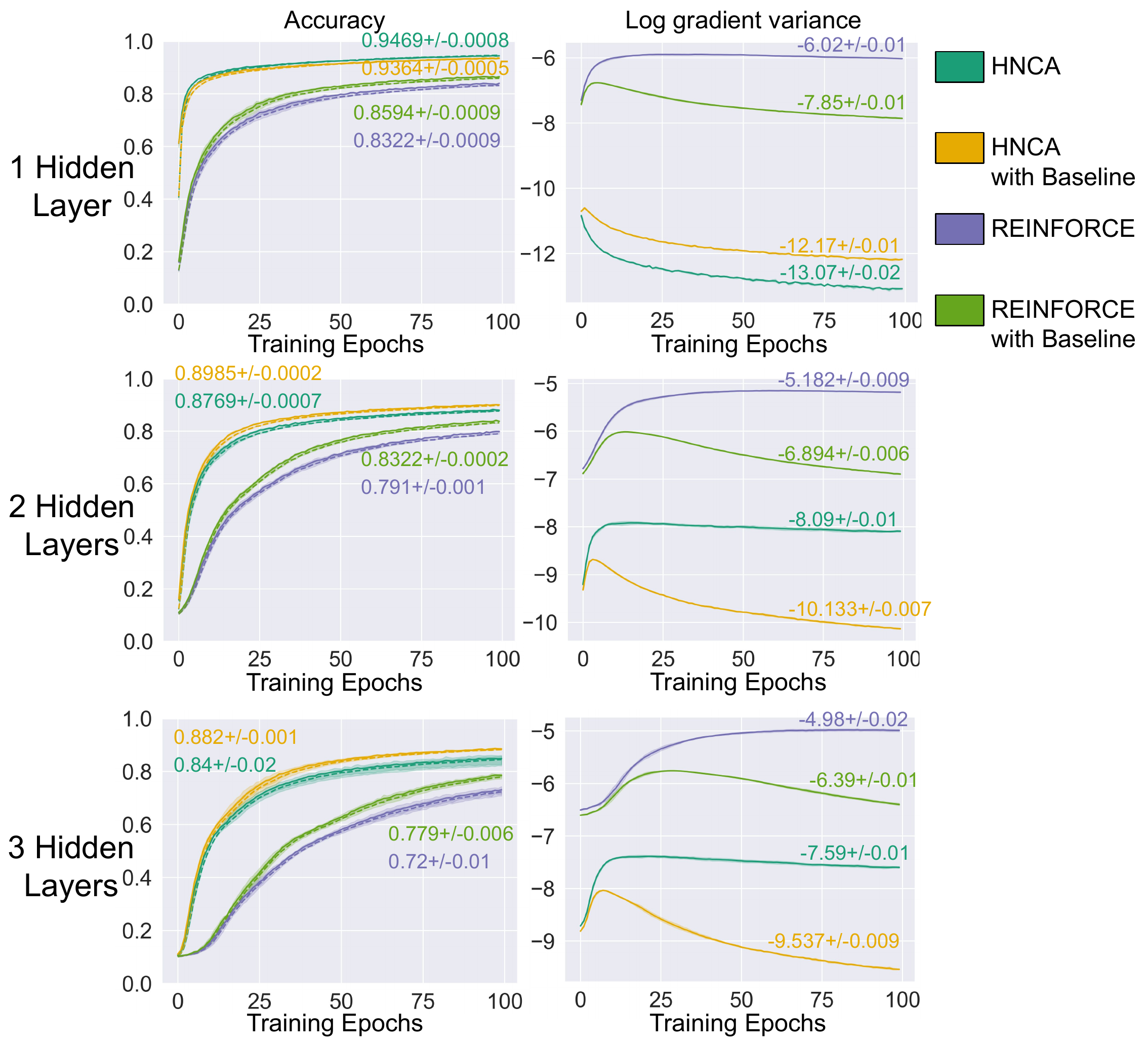}
    \caption{Training stochastic networks on a contextual bandit version of MNIST. Each line represents the average of 5 random seeds with error bars showing $95\%$ confidence interval. Final values (train accuracy for the left plots) at the end of training are written beside each line. The left column shows the online training accuracy (or equivalently the average reward) as a dotted line, and the test accuracy as a solid line (though they essentially overlap). The right column shows the natural logarithm of the mean gradient variance. Mean gradient variance is computed as the mean of the per-parameter empirical variance over examples in a training batch of $50$. We find that, for each network depth, HNCA drastically reduces gradient variance, resulting in significantly improved performance on this task.}
    \label{fig:contextual_bandit_plots}
\end{figure}
We evaluate HNCA against REINFORCE in terms of gradient variance and performance on a contextual bandit version of MNIST~\citep{lecun2010mnist}, with the standard train test split. Following~\citet{dong2020disarm}, input pixels are dynamically binarized, meaning that at each epoch they are randomly fixed to $0$ or $1$ with probability proportional to their intensity. For each training example, the model outputs a prediction and receives a reward of $1$ if correct and $0$ otherwise. We use a fully connected, feedforward network with 1, 2 or 3 hidden layers, each with 200 Bernoulli units, followed by a softmax output layer. We train using ADAM optimizer~\citep{kingma2014adam} with a learning rate fixed to $10^{-4}$ and batch-size of $50$ for $100$ epochs. Learning rate and layer size hyperparameters follow~\citet{dong2020disarm} for simplicity. We map the output of the Bernoulli units to one or negative one, instead of one or zero, as we found this greatly improved performance in preliminary experiments. We report results for HNCA and REINFORCE, both with and without an exponential moving average baseline subtracted from the reward. We use a discount rate of $0.99$ for the moving average.

Figure~\ref{fig:contextual_bandit_plots} shows the results, in terms of performance and gradient variance, for gradient estimates generated by HNCA and REINFORCE. We find that HNCA provides drastic improvement in terms of both gradient variance and performance over REINFORCE. Note that performance degrades with number of layers for both estimators, reflecting the increasing challenge of credit assignment. Subtracting a moving average baseline generally improves performance of both algorithms, except for HNCA in the single hidden layer case. The comparison between the two algorithms is qualitatively similar whether or not a baseline is used.


In Appendix~\ref{nonlinear_bandit}, we demonstrate that HNCA can also be used to efficiently train a stochastic layer as the final hidden layer of an otherwise deterministic network, this could be useful, for example, for learning a binary representation.

\section{Optimizing a Known Function}\label{extended}
In Section~\ref{bandit}, we introduced HNCA in a setting where the reward function was unknown, and dependent only on the input context and the output of the network as a whole. Here, we extend HNCA to optimize the expectation of a known function $f$, which may have direct dependence on every unit. We refer to this extension simply as $f$-HNCA. This setting is more in line with the setting explored by~\citet{titsias2015local}, and $f$-HNCA is distinguished from LEG mainly by its computationally efficient message passing implementation, which in turn facilitates its application to multi-layer stochastic networks.

We assume the function $f=\sum_if^i$ is factored into a number of \textit{function components} $f^i$, which we index by $i$ for convenience. This factored structure has two benefits, the first is computational. In particular, it will allow us to compute counterfactual values for each component with respect to changes to its input separately. The second is for variance reduction by realizing that we only need to assign credit to function components that lie downstream of the unit being credited. A similar variance reduction approach is also used by the NVIL algorithm of~\citet{mnih2014neural}. 

Each function component $f^i$ is a deterministic function of a subset of the outputs of units in the network, as well as possibly depending directly on some parameters. Thus, $f^i=f^i(\widetilde{\pa}(f^i);\theta^i)$, where $\theta^i$ is a set of real valued parameters which may overlap with the parameters $\theta_\Phi$ for some subset of units in the network, and $\widetilde{\pa}(f^i)$ is the set of nodes in the network which act as input to $f^i$. Formally, $f^i$ without arguments will refer to the random variable corresponding to the output of the associated function. We use the notation $\widetilde{\pa}$, distinct from $\pa$, to make it clear that function components are not considered nodes in the network.

The goal in this setting is to estimate the gradient of $\E[f]$, so that we can maximize it by gradient ascent. By linearity of expectation, we can define unbiased estimators for each $\frac{\partial\E[f^i]}{\partial\theta_\Phi}$ separately and sum over $i$ to get an unbiased estimator of the full gradient. 

\subsection{HNCA with a Known Function}\label{GHNCA}
We now discuss how to extend the HNCA estimator to construct an estimator of $\frac{\partial\E[f]}{\partial\theta_\Phi}$ for a particular unit $\Phi$ and function component in this setting. We begin by considering the gradient for a single function component $\frac{\partial\E[f^i]}{\partial\theta_\Phi}$.  First, note that we can break the gradient into indirect and direct dependence on $\theta_\Phi$:
\begin{equation}\label{gradient_decomp}
\frac{\partial\E[f^i]}{\partial\theta_\Phi}=\E\left[\frac{\partial\log(\pi_\Phi(\Phi|\pa(\Phi)))}{\partial\theta_\Phi}f^i\right]+\E\left[\frac{\partial f^i}{\partial\theta_\Phi}\right].
\end{equation}
The direct gradient $\frac{\partial f^i}{\partial\theta_\Phi}$ is zero unless $\theta_\Phi\in\theta^i$, in which case it can be computed directly given we assume access to $f^i$. From this point on, we will focus on the left expectation.

The main added complexity in estimating $\E\left[\frac{\partial\log(\pi_\Phi(\Phi|\pa(\Phi)))}{\partial\theta_\Phi}f^i\right]$, compared to the contextual bandit case, arises if $f^i$ has a direct functional dependence on $\Phi$. In this case we can no longer assume that $f^i$ is separated from $\Phi$ by $\mb(\Phi)$. Luckily, this is straightforward to patch. Let $f^i_\Phi(\phi)$ be the random variable defined by taking the function $f^i(\widetilde{\pa}(f^i);\theta^i)$ and substituting the specific value $\phi$ instead of the random variable $\Phi$ into the arguments while keeping all other $\widetilde{\pa}(f^i)$ equal to the associated random variables. By design, $f^i_\Phi(\phi)$ is independent of $\Phi$ given $\mb(\Phi)$, which allows us to define the following unbiased estimator for $\E\left[\frac{\partial\log(\pi_\Phi(\Phi|\pa(\Phi)))}{\partial\theta_\Phi}f^i\right]$ (see Appendix~\ref{Generalized_HNCA_derivation} for the full derivation):
\begin{equation}\label{Generalized_HNCA_i}
    \hat{G}^{\text{$f$-HNCA},i}_\Phi(\phi)\defeq\sum_\phi\rho_\Phi(\phi)\frac{\partial\pi_\Phi(\Phi|\pa(\Phi))}{\partial\theta_\Phi}f^i_\Phi(\phi),
\end{equation}
where $\rho_\Phi(\phi)$ is as in Equation~\ref{markov_blanket_conditional_expression}. As $\rho_\Phi(\phi)$ is defined with respect to $\ch(\Phi)$, this estimator is only applicable if $\Phi$ has children (i.e. $\ch(\Phi)\neq\emptyset$). In fact, even if $\Phi$ has children, we can ignore them if they have no downstream connection\footnote{More generally, if only a subset of $\ch(\Phi)$ lies in $\widetilde{\an}(f^i)$  we can replace $\ch(\Phi)$ in $\rho_\Phi(\phi)$ with $\ch^i(\Phi)=(\ch(\Phi)\cap\widetilde{\an}(f^i))$, but, we will not use this in our experiments in this work.} to $f^i$, as such children cannot influence $f^i$. Thus if $\ch(\Phi)\cap\widetilde{\an}(f^i)=\emptyset$ we instead define $\hat{G}^{\text{$f$-HNCA},i}_\Phi(\phi)\defeq\sum_\phi\frac{\partial\pi_\Phi(\Phi|\pa(\Phi))}{\partial\theta_\Phi}f^i_\Phi(\phi)$. In Appendix~\ref{Generalized_HNCA_low_variance}, we extend Theorem~\ref{reduced_variance} to apply to $f$-HNCA, showing that using $\hat{G}_{\Phi}^{\text{$f$-HNCA},i}$ results in a variance reduced estimator for $\E\left[\frac{\partial\log(\pi_\Phi(\Phi|\pa(\Phi)))}{\partial\theta_\Phi}f^i\right]$ compared to REINFORCE. The full $f$-HNCA gradient estimator is defined by summing up these components and accounting for any direct functional dependence of $f$ on network parameters:
\begin{equation}\label{generalized_HNCA}
\hat{G}_{\Phi}^{\text{$f$-HNCA}}\defeq\begin{multlined}[t]\sum_\phi\frac{\partial\pi_\Phi(\phi|\pa(\Phi))}{\partial\theta_\Phi}\Biggl(\rho_\Phi(\phi)\smashoperator{\sum_{i:\ch(\Phi)\cap\widetilde{\an}(f^i)\neq\emptyset}}f^i_\Phi(\phi)\hspace{15pt}+\\
\smashoperator{\sum_{i:\ch(\Phi)\cap\widetilde{\an}(f^i)=\emptyset}}f^i_\Phi(\phi)\hspace{10pt}\Biggr)+\sum_i\frac{\partial f^i}{\partial\theta_\Phi}.
\end{multlined}
\end{equation}


If $\Phi\not\in\widetilde{\an}(f^i)$ then $\frac{\partial\E[f^i]}{\partial\theta_\Phi}=\E\left[\frac{\partial f^i}{\partial\theta_\Phi}\right]$ as $\Phi$ cannot influence something with no downstream connection. Hence, in the two leftmost sums over $i$ in Equation~\ref{generalized_HNCA}, we implicitly only sum over $i$ such that $\Phi\in\widetilde{\an}(f^i)$. 

In addition to the efficiency of computing counterfactual probabilities, for $f$-HNCA, we have to consider the efficiency of computing counterfactual function components $f^i_\Phi(\phi)$ given $f^i$. For function components with no direct connection to a unit $\Phi$, this is trivial as $f^i_\Phi(\phi)=f^i$. If $f^i$ is directly connected, then implementing $f$-HNCA with efficiency similar to HNCA will require that we are able to compute $f^i_\Phi(\phi)$ from $f^i$ in constant time. This is the case if $f^i$ is a linear function followed by some activation. For example, functions of the form $f^i=\log(\sigma(\vec{\theta}\cdot\vec{x}+b))$ which will appear in the ELBO function used in our variational auto-encoder (VAE; \citet{kingma2013auto,rezende2014stochastic}) experiments. More algorithmic details can be found in Appendix~\ref{VAE_implementation}.

\subsection{Variational Auto-encoder Experiment}\label{VAE_experiments}
\definecolor{purple}{HTML}{7570b3}
\definecolor{orange}{HTML}{d95f02}
\definecolor{pink}{HTML}{e7298a}
\definecolor{green}{HTML}{66a61e}
\begin{figure}[!t]
    \centering
    \includegraphics[width=\columnwidth]{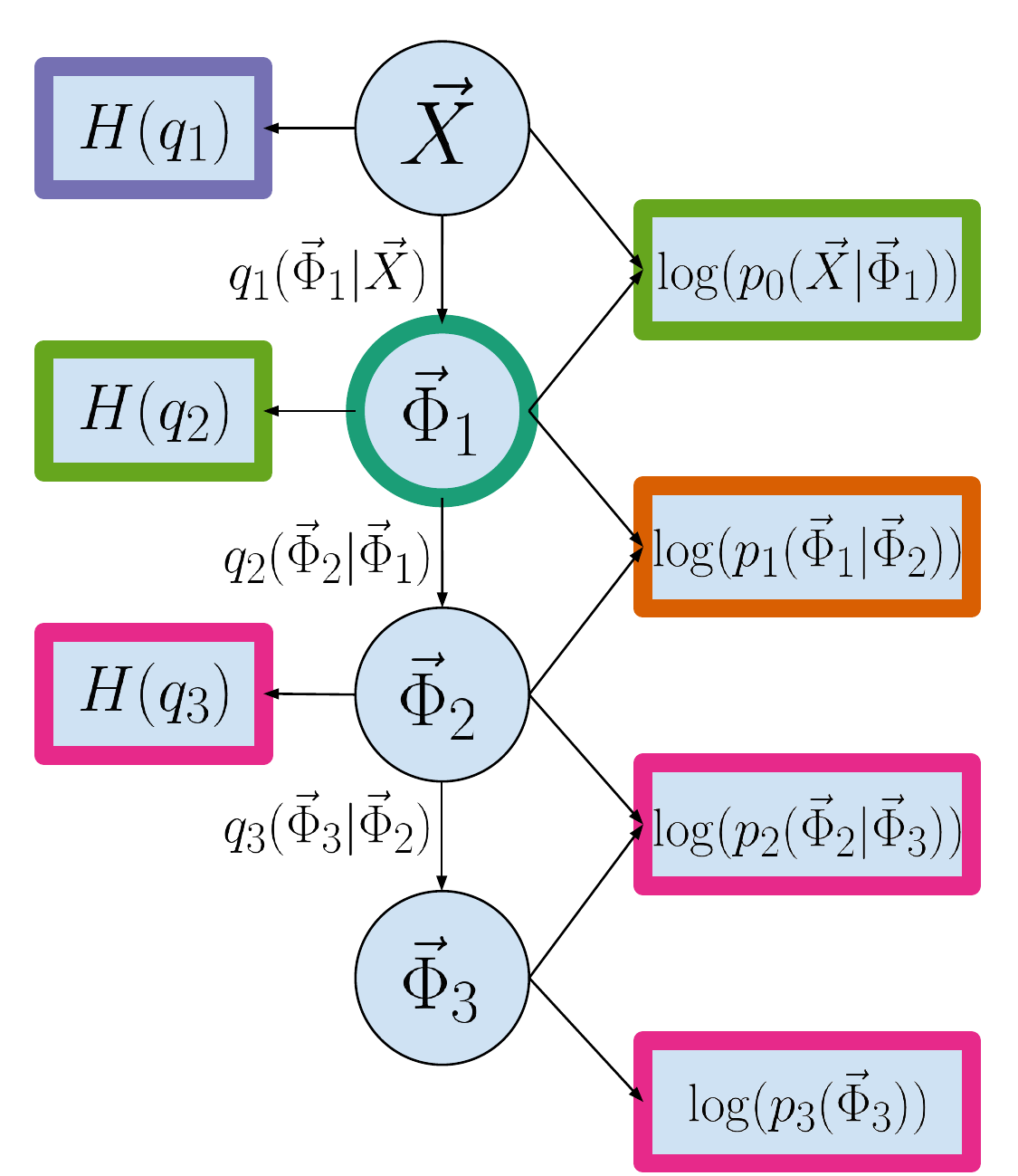}
    \caption{An illustration of the ELBO for a 3 layer discrete hierarchical VAE broken down into function components for $f$-HNCA. $\vec{X}$ is the input to be encoded, each additional circle is the latent state from a layer of the encoder network. Each rectangle is a set of function components which contribute to the ELBO. The parameters of the encoder are trained to maximize the ELBO by $f$-HNCA. Consider the $f$-HNCA gradient estimator for $\vec{\Phi}_1$. The function components $H(q_1)$, marked in \textcolor{purple}{purple} are upstream of $\vec{\Phi}_1$, however, $H(q_1)$ depends directly on $\theta_{q_1}$ and thus $\frac{\partial H(q_1)}{\partial \theta_{q_1}}$ is nonzero, so the entire contribution of $H(q_1)$ to the gradient estimate $\hat{G}^{\text{$f$-HNCA}}_\Phi$ will come from this gradient. The function components marked in \textcolor{green}{green} have only direct connection with $\vec{\Phi}_1$, so they will receive credit via $\hat{G}^{\text{$f$-HNCA},i}_\Phi(\phi)\defeq\sum_\phi\frac{\partial\pi_\Phi(\Phi|\pa(\Phi))}{\partial\theta_\Phi}f^i_\Phi(\phi)$. The function components marked in \textcolor{orange}{orange} have both direct connections and downstream connections mediated by $\vec{\Phi}_2$, so they will receive credit via Equation~\ref{Generalized_HNCA_i}. Finally, the variables marked in \textcolor{pink}{pink} have only mediated connections to $\vec{\Phi}_1$ through $\vec{\Phi}_2$, so $f^i_\Phi(\phi)=f^i$, the estimator for these variables essentially reduces to the original HNCA estimator defined in Equation~\ref{HNCA_estimator}.}
    \label{fig:generative_model}
\end{figure}
Here, we demonstrate how the $f$-HNCA approach described in Section~\ref{GHNCA} can be applied to the challenging task of training a discrete hierarchical VAE. Consider a VAE consisting of a generative model (decoder) $p$ and an approximate posterior (encoder) $q$, each of which consist of $L$ discrete stochastic layers. Samples $\vec{X}$ are generated by $p$ as
\begin{equation*}
    \vec{X}\sim p_0(\vec{X}|\vec{\Phi}_1), \vec{\Phi}_1\sim p_1(\vec{\Phi}_1|\vec{\Phi}_2), ..., \vec{\Phi}_L\sim p_L(\vec{\Phi}_L),
\end{equation*}
while $q$ approximates the posterior $\P(\vec{\Phi}_L|X)$ as a distribution which can be sampled as
\begin{equation*}
    \Phi_L\sim \begin{multlined}[t]q_L(\vec{\Phi}_L|\vec{\Phi}_{L-1}), \vec{\Phi}_{L-1}\sim q_{L-1}(\vec{\Phi}_{L-1}|\vec{\Phi}_{L-2}),...\\
    ,\vec{\Phi}_1\sim q_1(\vec{\Phi}_1|\vec{X}),
    \end{multlined}
\end{equation*}
where, each $p_i$ and $q_i$ represents a vector of Bernoulli distributions, each parameterized as a linear function of their input (except the prior $p_L(\vec{\Phi}_L)$ which takes no input, and is simply a vector of Bernoulli variables with learned means). Call the associated parameters $\theta_{p_i}$ and $\theta_{q_i}$ We can train such a VAE by maximizing a lower bound on the log-likelihood of the training data, usually referred to as the evidence lower bound (ELBO) which we can write as $\E[f_E]$ where
\begin{equation}\label{ELBO_reward}
\begin{multlined}[b]
f_E\defeq\log(p_0(\vec{X}|\vec{\Phi}_1))+\sum_{l=1}^{L-1}\log(p_l(\vec{\Phi}_l|\vec{\Phi}_{l-1}))+\\\log(p_L(\vec{\Phi}_L))+H(q_1(\cdot|\vec{X}))+\sum_{l=1}^{L-1}H(q_{l+1}(\cdot|\vec{\Phi}_l))
\end{multlined},
\end{equation}
where $H$ is the entropy of the distribution, and the expectation is taken with respect to the encoder $q$ and random samples $\vec{X}$. Each $\vec{\Phi}_i$ is sampled from the associated encoder $q_i$. Note that each term in Equation~\ref{ELBO_reward} is a sum over elements in the associated output vector, we can view each element as a particular function component $f^i$. The resulting compute graph is illustrated in Figure~\ref{fig:generative_model}. 


We compare $f$-HNCA with REINFORCE and several stronger methods for optimizing an ELBO of a VAE trained to generate MNIST digits. We focus on strong, unbiased, variance reduction techniques from the literature that do not require modifying the architecture or introduce significant additional hyperparameters. Since HNCA falls into this category, this allows for straightforward comparison without the additional nuance of architectural and hyperparameter choices. Specifically, we compare HNCA with REINFORCE leave one out (REINFORCE LOO; \citet{kool2019buy}) and DisARM~\citep{dong2020disarm}. Note that in the multi-layer case, both DisARM and REINFORCE LOO require sampling an additional partial forward pass beginning from each layer, which gives them a quadratic scaling in compute cost with the number of layers. By contrast, HNCA requires only a single forward pass and a backward pass of similar complexity.


Initially, we found that $f$-HNCA outperformed the other tested methods in the single layer discrete VAE case, but fell short in the multi-layer case. However, we found that a simple modification that subtracts a layer specific scalar baseline, similar to that used by~\citet{mnih2014neural}, significantly improved the performance of $f$-HNCA in the multi-layer case. Specifically, for each layer, we maintain a scalar running average of the sum of those components of f with mediated connections (those highlighted in pink and orange in Figure 2) and subtract it from the leftmost sum over i in Equation 6 to produce a centered learning signal.\footnote{Using such a baseline for components without mediated connections would analytically cancel.} We use a discount rate of $0.99$ for the moving average.\footnote{We used the first value we tried, we did not tune it.} We refer to this variant as $f$-HNCA with Baseline. We also tested subtracting a moving average of all downstream function components in REINFORCE to understand how much this change helps on its own. It's not obvious how to implement such a running average baseline for the other tested methods given they already utilize alternative methods to center the learning signal, thus a naive moving average baseline would have expectation zero.

As in Section~\ref{bandit_experiments}, we use dynamic binarization and train using ADAM optimizer with learning rate $10^{-4}$ and batch-size $50$. Following~\citet{dong2020disarm}, our decoder and encoder each consist of a fully connected, stochastic feedforward neural network with 1, 2 or 3 layers, each hidden layer has 200 Bernoulli units. We train for $840$ epochs, approximately equivalent to the $10^6$ updates used by~\citet{dong2020disarm}. For consistency with prior work, we use Bernoulli units with a zero-one output. For all methods, we train each unit based on downstream function components, as opposed to using the full function $f$. See Appendix~\ref{VAE_exp_details} for more implementation details.


Figure~\ref{fig:generative_model_plots}, shows the results in terms of ELBO and gradient variance, for gradient estimates generated by $f$-HNCA and the other methods tested. As in the contextual bandit case, we find that $f$-HNCA provides drastic improvement over REINFORCE. $f$-HNCA also provides a significant improvement over all other methods for the single-layer discrete VAE, but underperforms the other strong methods in the multi-layer case. On the other hand, $f$-HNCA with Baseline significantly improves on the other tested methods in all cases. REINFORCE with baseline outperforms ordinary $f$-HNCA in the multi-layer cases. Hence, this baseline subtraction is a fairly powerful variance reduction technique for REINFORCE, with strong complementary benefits with $f$-HNCA. In Appendix~\ref{test_set_ELBOs}, we additionally report multi-sample test-set ELBOs for the final trained networks, which reflect the same performance ordering as the training set ELBOs. In Appendix~\ref{ablation}, we perform an ablation experiments on $f$-HNCA with Baseline and find that the choice of whether to exclude children when $\ch(\Phi)\cap\widetilde{\an}(f^i)=\emptyset$ has a significant performance impact, while the additional impact of excluding upstream function components is fairly minimal.

\begin{figure}[htb]
    \includegraphics[width=\columnwidth]{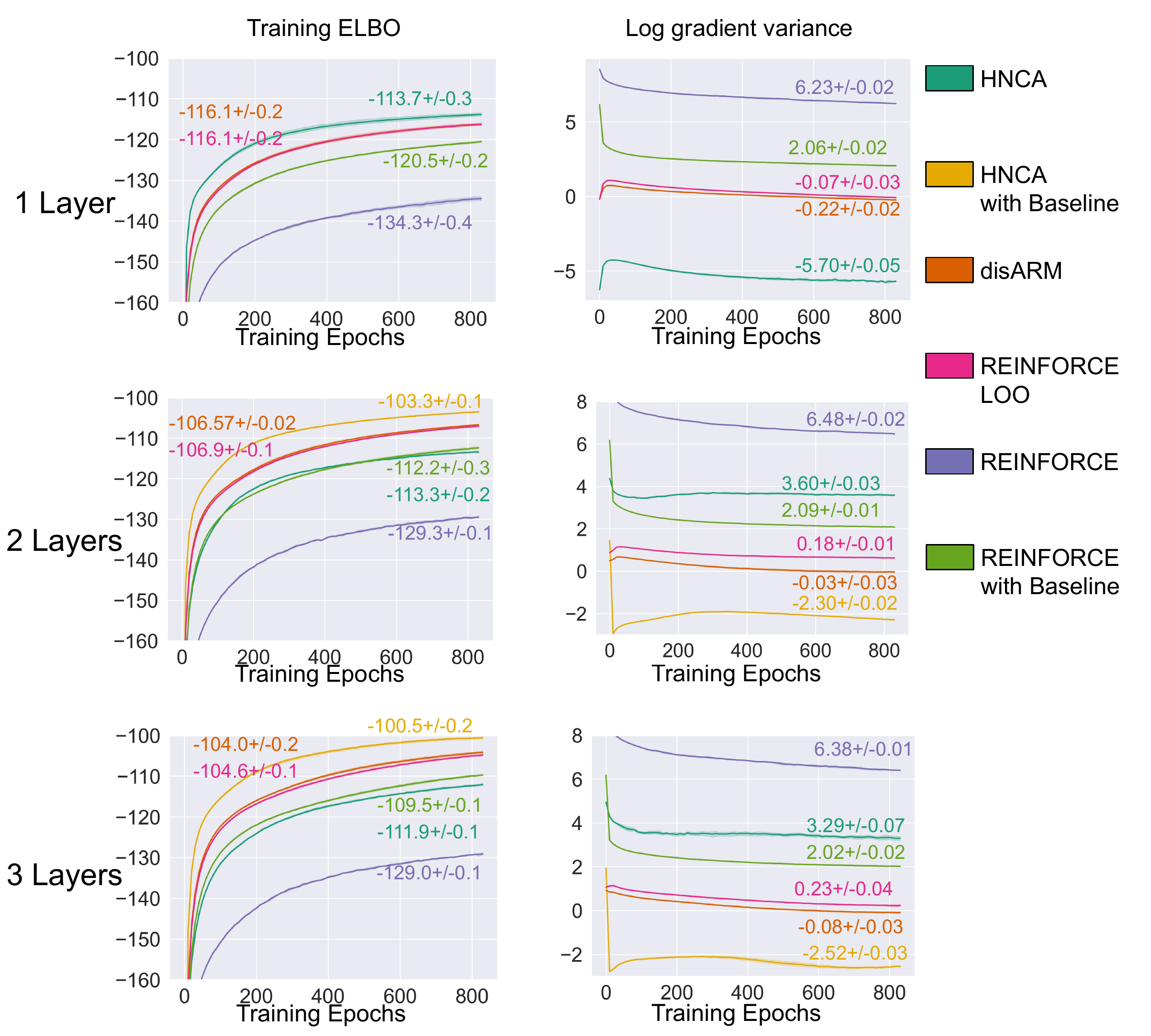}
        \caption{Training discrete VAEs to generate MNIST digits. Each line represents the average of 5 random seeds with error bars showing $95\%$ confidence interval. Final values at the end of training are written near each line in matching color. The left column shows the online training ELBO. The right column shows the natural logarithm of the mean encoder gradient variance. Mean gradient variance is computed as the mean over parameters and batches of the per-parameter empirical variance over examples in a training batch of $50$. $f$-HNCA outperforms all other tested methods in the single-layer case, but underperforms in the multi-layer cases. $f$-HNCA with Baseline outperforms the other methods in the multi-layer case. $f$-HNCA with baseline is excluded from the single layer results as there are no mediated connections.}
    \label{fig:generative_model_plots}
\end{figure}

\section{Discussion and Conclusion}
We introduced HNCA, an algorithm for gradient estimation in networks of discrete stochastic units. HNCA is inspired by Hindsight Credit Assignment~\citep{harutyunyan2019hindsight}, and can be seen as an instance of Local Expectation Gradients, extending the work of~\citet{titsias2015local} by providing a computationally efficient message passing algorithm and extension to multi-layer networks of stochastic units. Our computational efficient approach directly addresses concerns in the literature that LEG is inherently computationally expensive~\citep{tucker2017rebar,mnih16vimco}. We prove that HNCA is unbiased, and that it reduces variance compared to REINFORCE. Empirically, we show that HNCA outperforms strong methods for training a single-layer Bernoulli VAE, and when subtracting a simple moving average baseline also outperforms the same methods for the case of a multi-layer Hierarchical VAE.

It's worth highlighting that efficient implementation of HNCA is predicated on the ability to efficiently compute counterfactual probabilities or function components when a single input is changed. This is not always possible, for example, if $f$ is the result of a multi-layer deterministic network. An example of this situation is the nonlinear discrete VAE architecture explored by \citet{dong2020disarm} and \citet{yin2019arm} where the encoder and decoder are nonlinear networks with a single stochastic Bernoulli layer at the outputs. However, as we show in Appendix~\ref{nonlinear_bandit}, HNCA can be used to train a final Bernoulli hidden layer at the end of a nonlinear network.

In addition to optimizing a known function of the output of a stochastic network, we show in Section~\ref{bandit} that HNCA can be applied to train the hidden layers of a multi-layer discrete network in an online learning setting with unknown reward function. REINFORCE LOO and DisARM, which rely on the ability to evaluate the reward function multiple times for a single training example, cannot.

Future work could explore combining HNCA with other methods for complimentary benefits. One could also explore extending HNCA to propagate credit multiple steps which would presumably allow further variance reduction, but presents challenges as the relationships between more distant nodes in the network becomes increasingly complex.


HNCA provides insight into the challenges of credit assignment in discrete stochastic compute graphs, which has the potential to have an impact on future approaches. 

\section*{Acknowledgments}
The author thanks Rich Sutton, Matt Taylor and Tian Tian for useful conversations, and anonymous reviewers for useful feedback. I also thank the Natural Sciences and Engineering Research Council of Canada and Alberta Innovates for providing funding for this work.
\bibliography{refs}

\newpage
\appendix
\onecolumn
\begin{center}{\Huge \textbf{Appendix}}\end{center}
\section{The Local REINFORCE Estimator is Unbiased}\label{local_REIN_unbiased}
Here, we show that the local REINFORCE estimator $\hat{G}^{\text{RE}}_{\Phi}=\frac{\partial\log(\pi_\Phi(\Phi|\pa(\Phi)))}{\partial\theta_\Phi}R$ is an unbiased estimator of the gradient of the expected reward with respect to $\theta_\Phi$.
\begin{align*}
     &\E[\hat{G}^{\text{RE}}_{\Phi}]=\E\left[\frac{\partial\log(\pi_\Phi(\Phi|\pa(\Phi)))}{\partial\theta_\Phi}R\right]\\
     &\stackrel{(a)}{=}\sum_{b}\P(\pa(\Phi)=b)\sum_\phi\pi_\Phi(\phi|b)\frac{\partial\log(\pi_\Phi(\phi|b))}{\partial\theta_\Phi}\E\left[R\middle|\pa(\Phi)=b,\Phi=\phi\right]\\
     &\stackrel{(b)}{=}\sum_{b}\P(\pa(\Phi)=b)\sum_\phi\frac{\partial\pi_\Phi(\phi|b)}{\partial\theta_\Phi}\E\left[R\middle|\pa(\Phi)=b,\Phi=\phi\right]\\
     &\stackrel{(c)}{=}
     \frac{\partial}{\partial\theta_\Phi}\sum_{b}\P(\pa(\Phi)=b)\sum_\phi\pi_\Phi(\phi|b)\E\left[R\middle|\pa(\Phi)=b,\Phi=\phi\right]\\
     &=\frac{\partial\E[R]}{\partial\theta_\Phi},
\end{align*}

where $(a)$ expands the expectation over $\pa(\Phi)$ and $\Phi$, $(b)$ rewrites the log gradient, and $(c)$ follows from the fact that the probability of the parents of $\Phi$, $\P(\pa(\Phi)=b))$, does not depend on the parameters $\theta_\Phi$ controlling $\Phi$ itself, nor does the expected reward conditioned on $\Phi$ and $\pa(\Phi)$. 

\section{Derivation of Expression for Conditional Probability of Unit Output Conditioned on a Markov Blanket}\label{markov_blanket_conditional}
Here, we prove Equation~\ref{markov_blanket_conditional_expression}, that is
\begin{equation*}
    \P(\Phi=\phi|\mb(\Phi))=\frac{\pi_\Phi(\phi|\pa(\Phi))\prod\limits_{C\in\ch(\Phi)}\pi_C(C|\pa(C)\setminus\Phi,\Phi=\phi)}{\sum_{\phi'}\pi_\Phi(\phi'|\pa(\Phi))\prod\limits_{C\in\ch(\Phi)}\pi_C(C|\pa(C)\setminus\Phi,\Phi=\phi')},
\end{equation*}
which is used in deriving the HNCA gradient estimator. In doing so, we will use Theorem 1 from Section 4 of~\citet{pearl1988probabilistic}, restated here in our notation for convenience:
\begin{theorem}[Theorem 1~\citep{pearl1988probabilistic}]\label{markov_blanket_theorem}
Let $X$ be a random variable in a Bayesian network. Let $\neg X$ represent the set of all random variables in the network besides $X$. Then:
\begin{equation*}
    \P(X=x|\neg X)=\alpha\P(X=x|\pa(X))\prod_{C\in\ch(X)}\P(C|\pa(C)\setminus X,X=x),
\end{equation*}
where $\alpha$ is a normalizing factor which does not depend on $x$.
\end{theorem}
Using this theorem, we can compute $\P(\Phi=\phi|\mb(\Phi))$ as follows:
\begin{align*}
    \P(\Phi=\phi|\mb(\Phi))&\stackrel{(a)}{=}\P(\Phi=\phi|\neg\Phi)\\
    &\stackrel{(b)}{=}\alpha\P(\Phi=\phi|\pa(\Phi))\prod_{C\in\ch(\Phi)}\P(C|\pa(C)\setminus\Phi,\Phi=\phi)\\
    &\stackrel{(c)}{=}\frac{\P(\Phi=\phi|\pa(\Phi))\prod\limits_{C\in\ch(\Phi)}\P(C|\pa(C)\setminus\Phi,\Phi=\phi)}{\sum\limits_{\phi'}\P(\Phi=\phi'|\pa(\Phi))\prod\limits_{C\in\ch(\Phi)}\P(C|\pa(C)\setminus\Phi,\Phi=\phi')}\\
    &\stackrel{(d)}{=}\frac{\pi_\Phi(\phi|\pa(\Phi))\prod\limits_{C\in\ch(\Phi)}\pi_C(C|\pa(C)\setminus\Phi,\Phi=\phi)}{\sum_{\phi'}\pi_\Phi(\phi'|\pa(\Phi))\prod\limits_{C\in\ch(\Phi)}\pi_C(C|\pa(C)\setminus\Phi,\Phi=\phi')},\\
\end{align*}
where $(a)$ follows from the fact that $\mb(\Phi)=\{\ch(\Phi),\pa(\Phi),\pa(\ch(\Phi))\setminus\Phi\}$ is a \textit{minimal} Markov blanket for $\Phi$ and hence $\Phi$ is independent of all other variables in the network given $\mb(\Phi)$, $(b)$ follows from Theorem~\ref{markov_blanket_theorem}, $(c)$ simply makes the normalizing factor $\alpha$ explicit and $(d)$ uses the fact that $\P(\Phi=\phi|\pa(\Phi))=\pi_\Phi(\phi|\pa(\Phi))$.

\section{The HNCA Gradient Estimator has Lower Variance than the REINFORCE Estimator}\label{HNCA_action_value_low_var}
Here, we provide the proof of Theorem~\ref{reduced_variance}.
\begingroup
\def\thetheorem{\ref{reduced_variance}}
\begin{theorem}
Recall that
\begin{equation*}
\hat{G}^{\text{RE}}_{\Phi}\defeq\frac{\partial\log(\pi_\Phi(\Phi|\pa(\Phi)))}{\partial\theta_\Phi}R
\end{equation*}
and 
\begin{equation*}
 \hat{G}^{\text{HNCA}}_{\Phi}=\sum_\phi\frac{\prod\limits_{C\in\ch(\Phi)}\pi_C(C|\pa(C)\setminus\Phi,\Phi=\phi)}{\sum\limits_{\phi'}\pi_\Phi(\phi'|\pa(\Phi))\prod\limits_{C\in\ch(\Phi)}\pi_C(C|\pa(C)\setminus\Phi,\Phi=\phi')}\frac{\partial\pi_\Phi(\phi|\pa(\Phi))}{\partial\theta_\Phi}R,
\end{equation*}
then $\Var(\hat{G}^{\text{HNCA}}_{\Phi})\leq \Var(\hat{G}^{\text{RE}}_{\Phi})$, where $\Var(\vec{X})$ stand for the elementwise variance of random vector $\vec{X}$, and the inequality holds elementwise.
\end{theorem}
\addtocounter{theorem}{-1}
\endgroup
\begin{proof}
The proof follows from applying the law of total variance elementwise. From the derivation in Section~\ref{bandit} we know that
\begin{equation*}
    \hat{G}^{\text{HNCA}}_{\Phi}=\E\left[\hat{G}^{\text{RE}}_{\Phi}\middle|\mb(\Phi),R\right].
\end{equation*}
Now apply the law of total variance to rewrite the variance of the REINFORCE estimator as follows:
\begin{align*}
    \Var(\hat{G}^{\text{RE}}_{\Phi})&=\begin{multlined}[t]
    \E\left[\Var\left(\hat{G}^{\text{RE}}_{\Phi}\middle|\mb(\Phi),R\right)\right]+\Var\left(\E\left[\hat{G}^{\text{RE}}_{\Phi}\middle|\mb(\Phi),R\right]\right)
    \end{multlined}\\
    &\geq \Var\left(\E\left[\hat{G}^{\text{RE}}_{\Phi}\middle|\mb(\Phi),R\right]\right)\\
    &=\Var(\hat{G}^{\text{HNCA}}(\Phi)).
\end{align*}
\end{proof}

\section{HNCA for Softmax Output layer of Contextual Bandit Experiments}\label{softmax_alg}
For the $\Phi=A$, corresponding to the softmax output layer, computing a counterfactual probability $\pi_\Phi(\Phi|\pa(\Phi)\setminus B, B=b)$, will require $\bigO(N_A)$ time (where $N_A$ is the number of possible actions), instead of constant time. This can be seen by noting that we can easily compute the counterfactual logit corresponding to each action in constant time, but to compute the probability of any given action we must compute counterfactual logits for all actions. Hence, to compute counterfactual probabilities for each parent of the output unit will require $\bigO(NN_A|\pa(A)|)$, where again $N$ is number of possible outputs for each parent, assumed the same across parents. Note that this is again $N=2$ times the complexity of the forward pass if all the parents are Bernoulli units. Again, this can be reduced to $N-1=1$ by reusing the value computed in the forward pass.

Algorithm~\ref{fig:softmax_alg} provides an efficient pseudocode implementation for the softmax output unit used in our contextual bandit experiments. Note that the output unit itself uses the REINFORCE estimator in its update, as it has no children, which precludes the use of HNCA. Nonetheless, the output unit still needs to provide information to its parents, which do use HNCA.

If the entire network consisted of softmax units, each with $N$ output choices, we can see from the above discussion that computing all counterfactual probabilities for each parent would require $\bigO(N^2\sum_\Phi|\pa(\Phi)|)$. On the other hand, the forward pass in this case only requires $\bigO(N\sum_\Phi|\pa(\Phi)|)$. Hence, HNCA would add a factor of $N$ overhead in this case compared to the forward pass. However, it's worth noting that applying the biased straight-through estimator in the softmax case, as is done for example by~\citet{hafner2021mastering}, in principle suffers the same $N$ overhead for the backward pass. This is because while the forward pass simply needs to pass a single output for each node, the backward pass operates as if a size $N$ vector of probabilities had been passed, which blows up the input size by a factor of $N$.
\begin{figure}[htb]
\begin{algorithm}[H]
\begin{algorithmic}[1]
\STATE Receive $\vec{x}$ from parents\\
\STATE $\vec{l}=\Theta\vec{x}+\vec{b}$\\
\STATE $\vec{p}=\frac{\exp{\vec{l}}}{\sum_i\exp{\vec{l}[i]}}$\\
\STATE Output $\phi\sim \vec{p}$\\
\STATE Receive $R$ from environment\\
\FOR{all $i$}
\STATE $L_1[i]=\vec{l}+\Theta[i]\odot(1-\vec{x})$\\
\STATE $L_0[i]=\vec{l}-\Theta[i]\odot\vec{x}$
\ENDFOR
\STATE $\vec{p}_1=\frac{\exp{L_1[\phi]}}{\sum_i\exp{L_1[i]}}$\\
\STATE $\vec{p}_0=\frac{\exp{L_0[\phi]}}{\sum_i\exp{L_0[i]}}$\\
\STATE Pass $\vec{p}_1,\vec{p}_0,R$ to parents\\
\FOR{all $i$}
\STATE $\Theta[i]=\Theta[i]+\alpha\vec{x}(\ind(\phi=i)-\vec{p}[i])R$\\
\STATE $b[i]=b[i]+\alpha(\ind(\phi=i)-\vec{p}[i])R$
\ENDFOR
\end{algorithmic}
\captionsetup{labelformat=empty}
\caption{HNCA (Softmax output unit)}\label{fig:softmax_alg}
\end{algorithm}
\addtocounter{algorithm}{-1}
\captionof{algorithm}{Efficient implementation of HNCA message passing for a softmax output unit in a contextual bandit setting. Lines 1-4 implement the forward pass, in this case producing an integer $\phi$ corresponding to the possible actions. Lines 6-11 compute counterfactual probabilities of the given output class conditional on fixing the value of each parent. Note that $\Theta[i]$ refers to the $i_{th}$ row of the matrix $\Theta$. In this case, computing these counterfactual probabilities requires computation on the order of the number of parents, times the number of possible actions. Line 12 passes the necessary information back to the parents. Lines 13-16 update the parameters according to $\hat{G}^{\text{RE}}_{\Phi}$.}
\end{figure}

\section{HNCA to Train a Final Bernoulli Hidden Layer in a Nonlinear Network}\label{nonlinear_bandit}
\begin{figure}[htb]
    \centering
    \includegraphics[width=\columnwidth]{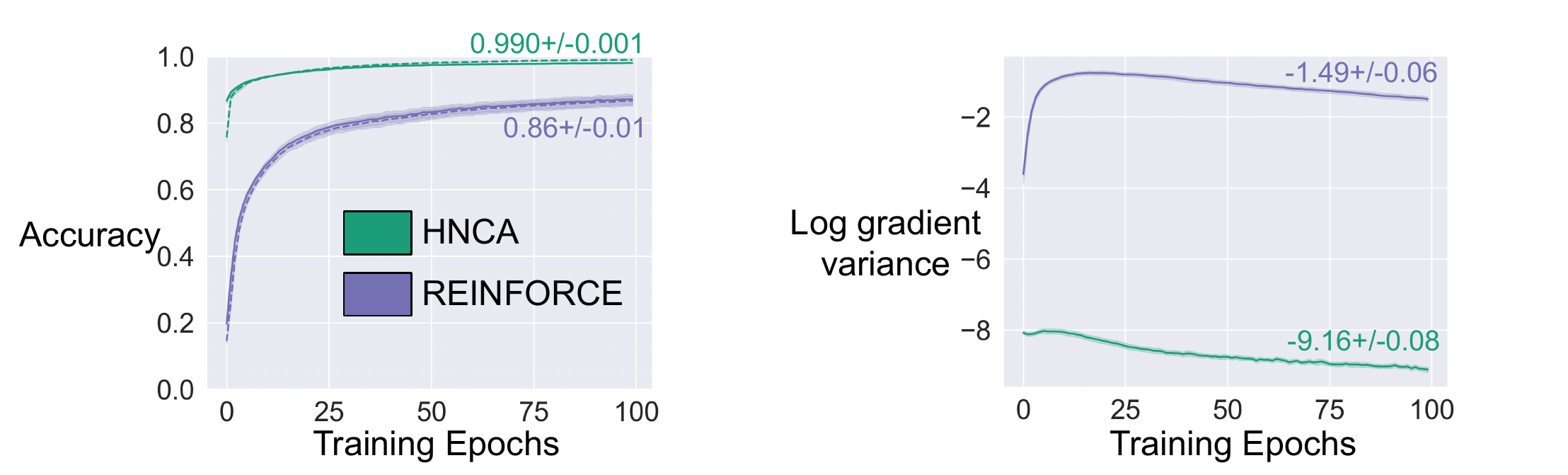}
    \caption{Training stochastic networks on a contextual bandit version of MNIST with two deterministic convolutional layer forming the input to a single Bernoulli hidden layer. Each line represents the average of 20 random seeds with error bars showing $95\%$ confidence interval. Final values at the end of training (train accuracy for the left plot) are written near each line in matching color. The top row shows the online training accuracy (or equivalently the average reward) as a dotted line, and the test accuracy as a solid line. The bottom row shows the natural logarithm of the mean gradient variance. Mean gradient variance is computed as the mean of the per-parameter empirical variance over examples in a training batch of $50$. HNCA significantly outperforms REINFORCE in this setting.}
    \label{fig:nonlinear_contextual_bandit_plots}
\end{figure}

Here, we provide a simple demonstration of using HNCA to train a Bernoulli layer as the last hidden layer of a nonlinear network. The task is the same contextual bandit version of MNIST outlined in Section~\ref{bandit_experiments}. The architecture consists of a two convolutional layers with 16 channels each, followed by ReLU activation which then feeds into a layer of 200 Bernoulli units, and finally a softmax output. To compute the HNCA estimator in this case we again use Equation~\ref{HNCA_estimator},  but now the gradients $\frac{\partial\pi_\Phi(\phi|X)}{\partial\theta_\Phi}$ are computed by backprop and summed over units when parameters are shared between them. More precisely, define $q_0^j$, $q_1^j$ and $\bar{q}^j$ as in Algorithm~\ref{bernoulli_alg} but with an additional index $j$ indicating the specific unit in the Bernoulli layer. The HNCA estimator can then efficiently implemented in an automatic differentiation framework by defining the following loss:
\begin{equation*}
    \mathcal{L}=-R\sum_j\text{SG}\left(\frac{q_1^j-q_0^j}{\bar{q}^j}\right)\pi_{\Phi^j}(\phi|X),
\end{equation*}
where in this case $\pi_{\Phi^j}(\phi|X)$ is the policy of unit $j$, and has a differentiable, nonlinear dependence on the context with arbitrary parameter sharing between units. SG stands for stop gradient, indicating that gradients are not propagated through $\frac{q_1^j-q_0^j}{\bar{q}^j}$. Computing the gradient of this loss function gives us the HNCA gradient estimator for this case, that is $\hat{G}^{HNCA}=\frac{\partial\mathcal{L}}{\partial\theta}$. The softmax output unit still implements Algorithm~\ref{fig:softmax_alg}.

We again compare against REINFORCE. As in Section~\ref{bandit_experiments} we map the output of the Bernoulli units to $-1$ or $1$. The results are shown in Figure~\ref{fig:nonlinear_contextual_bandit_plots} where we see that HNCA again provides a significant benefit over REINFORCE in this setting.

\newpage
\section{Derivation of $f$-HNCA Estimator}\label{Generalized_HNCA_derivation}
In this section, we elaborate on how to derive the $f$-HNCA gradient estimator $\hat{G}^{\text{$f$-HNCA},i}_\Phi(\phi)$. Recall that we defined $f^i_\Phi(\phi)$ as the random variable defined by taking the function $f^i(\widetilde{\pa}(f^i);\theta^i)$ and substituting the specific value $\phi$ instead of the random variable $\Phi$ into the arguments while keeping all other $\widetilde{\pa}(f^i)$ equal to the associated random variables. With this definition, we can express $\E\left[\frac{\partial\log(\pi_\Phi(\Phi|\pa(\Phi)))}{\partial\theta_\Phi}f^i\right]$ as follows:
\begin{align*}
    \E\left[\frac{\partial\log(\pi_\Phi(\Phi|\pa(\Phi)))}{\partial\theta_\Phi}f^i\right]&\stackrel{(a)}{=}\E\left[\E\left[\frac{\partial\log(\pi_\Phi(\Phi|\pa(\Phi)))}{\partial\theta_\Phi}f^i\middle|\mb(\Phi),\widetilde{\pa}(f^i)\setminus\Phi\right]\right]\\
    &\stackrel{(b)}{=}\E\left[\E\left[\sum_{\phi}\ind(\Phi=\phi)\frac{\partial\log(\pi_\Phi(\phi|\pa(\Phi)))}{\partial\theta_\Phi}f^i_\Phi(\phi)\middle|\mb(\Phi),\widetilde{\pa}(f^i)\setminus\Phi\right]\right]\\
    &\stackrel{(c)}{=}\E\left[\sum_{\phi}\E\left[\ind(\Phi=\phi)\middle|\mb(\Phi),\widetilde{\pa}(f^i)\setminus\Phi\right]\frac{\partial\log(\pi_\Phi(\phi|\pa(\Phi)))}{\partial\theta_\Phi}f^i_\Phi(\phi)\right]\\
    &\stackrel{(d)}{=}\E\left[\sum_{\phi}\E\left[\ind(\Phi=\phi)\middle|\mb(\Phi)\right]\frac{\partial\log(\pi_\Phi(\phi|\pa(\Phi)))}{\partial\theta_\Phi}f^i_\Phi(\phi)\right]\\
    &=\E\left[\sum_{\phi}\P(\Phi=\phi|\mb(\Phi))\frac{\partial\log(\pi_\Phi(\phi|\pa(\Phi)))}{\partial\theta_\Phi}f^i_\Phi(\phi)\right]\numberthis\label{GHNCA_derivation_1},
\end{align*}
where $(a)$ applies the law of total expectation, $(b)$ follows because the indicator function is zero except where the summand equals the expression from the previous line, $(c)$ moves deterministic quantities out of the inner expectation, and $(d)$ exploits the fact that $\Phi$ is independent of $\widetilde{\pa}(f^i)\setminus\Phi$ given $\mb(\Phi)$. From here, as in Section~\ref{bandit}, we substitute Equation~\ref{markov_blanket_conditional_expression} into the expression within the expectation to get the following unbiased estimator for $\E\left[\frac{\partial\log(\pi_\Phi(\Phi|\pa(\Phi)))}{\partial\theta_\Phi}f^i\right]$:
\begin{equation*}
\hat{G}^{\text{$f$-HNCA},i}_\Phi(\phi)\defeq\sum_\phi\rho_\Phi(\phi)\frac{\partial\pi_\Phi(\Phi|\pa(\Phi))}{\partial\theta_\Phi}f^i_\Phi(\phi),
\end{equation*}
where $\rho_\Phi(\phi)=\frac{\prod\limits_{C\in\ch(\Phi)}\pi_C(C|\pa(C)\setminus\Phi,\Phi=\phi)}{\sum\limits_{\phi'}\pi_\Phi(\phi'|\pa(\Phi))\prod\limits_{C\in\ch(\Phi)}\pi_C(C|\pa(C)\setminus\Phi,\Phi=\phi')}$. In the case where $\ch(\Phi)\cap\widetilde{\an}(f^i)=\emptyset$ it's not necessary to propagate credit from the children, $\ch(\Phi)$, as they cannot influence the reward. In this case, we instead use a simpler estimator derived as follows:
\begin{align*}
    \E\left[\frac{\partial\log(\pi_\Phi(\Phi|\pa(\Phi)))}{\partial\theta_\Phi}f^i\right]&\stackrel{(a)}{=}\E\left[\E\left[\frac{\partial\log(\pi_\Phi(\Phi|\pa(\Phi)))}{\partial\theta_\Phi}f^i\middle|\pa(\Phi),\widetilde{\pa}(f^i)\setminus\Phi\right]\right]\\
    &\stackrel{(b)}{=}\E\left[\E\left[\sum_{\phi}\ind(\Phi=\phi)\frac{\partial\log(\pi_\Phi(\phi|\pa(\Phi)))}{\partial\theta_\Phi}f^i_\Phi(\phi)\middle|\pa(\Phi),\widetilde{\pa}(f^i)\setminus\Phi\right]\right]\\
    &\stackrel{(c)}{=}\E\left[\sum_{\phi}\E\left[\ind(\Phi=\phi)\middle|\pa(\Phi)\right]\frac{\partial\log(\pi_\Phi(\phi|\pa(\Phi)))}{\partial\theta_\Phi}f^i_\Phi(\phi)\right]\\
    &=\E\left[\sum_{\phi}\pi_\Phi(\phi|\pa(\Phi))\frac{\partial\log(\pi_\Phi(\phi|\pa(\Phi)))}{\partial\theta_\Phi}f^i_\Phi(\phi)\right]\\
    &=\E\left[\sum_{\phi}\frac{\partial\pi_\Phi(\phi|\pa(\Phi))}{\partial\theta_\Phi}f^i_\Phi(\phi)\right]\numberthis\label{GHNCA_derivation_2},
\end{align*}
where $(a)$ applies the law of total expectation, $(b)$ follows because the indicator function is zero except where the summand equals the expression from the previous line, $(c)$ exploits the fact that $\Phi$ is independent of $\widetilde{\pa}(f^i)\setminus\Phi$ given $\pa(\Phi)$ due to the assumption $\ch(\Phi)\cap\widetilde{\an}(f^i)=\emptyset$. The final expression within the expectation gives us the unbiased estimator

\begin{equation*}
\hat{G}^{\text{$f$-HNCA},i}_\Phi(\phi)\defeq\sum_\phi\frac{\partial\pi_\Phi(\phi|\pa(\Phi))}{\partial\theta_\Phi}f^i_\Phi(\phi).
\end{equation*}
In our experiments we only distinguish the cases where $\ch(\Phi)\cap\widetilde{\an}(f^i)=\emptyset$ and $\ch(\Phi)\cap\widetilde{\an}(f^i)\neq\emptyset$. However, as alluded to in Section~\ref{GHNCA}, if only a subset of $\ch(\Phi)$ lies in $\widetilde{\an}(f^i)$  we can replace $\ch(\Phi)$ in $\rho_\Phi(\phi)$ with $\ch^i(\Phi)=(\ch(\Phi)\cap\widetilde{\an}(f^i))$. To see that this is the case, it suffices to note that if a particular child $C$ has no downstream connections to a particular function component $f^i$, then $\frac{\partial\E[f^i]}{\partial\theta_\Phi}$ must be the same in a new network with the connection from $\Phi$ to $C$ severed as in the original network.

\section{Efficient Implementation of $f$-HNCA}\label{VAE_implementation}
\begin{figure}[htb]
\begin{algorithm}[H]
\begin{algorithmic}[1]
\STATE Receive $\vec{x}$ from parents\\
\STATE $l=\vec{\theta}\cdot\vec{x}+b$\\
\STATE $f=\eta(l)$\\
\STATE $\vec{l}_1=l+\vec{\theta}\odot(1-\vec{x})$\\
\STATE $\vec{l}_0=l-\vec{\theta}\odot\vec{x}$\\
\STATE $\vec{f}_1=\eta(\vec{l}_1)$\\
\STATE $\vec{f}_0=\eta(\vec{l}_0)$\\
\STATE Pass $\vec{f}_1,\vec{f}_0$, $f$ to parents\\
\end{algorithmic}
\captionsetup{labelformat=empty}
\caption{$f$-HNCA algorithm for Linear Function Components}\label{reward_var_alg}
\end{algorithm}
\addtocounter{algorithm}{-1}
\captionof{algorithm}{Efficient implementation of $f$-HNCA for a function component which consists of a linear function of its inputs followed by an arbitrary activation $\eta$. Inputs are assumed to be Bernoulli. The forward pass in lines 1-3 takes input from the parents and uses it to compute the function component $R$. Line 4-7 use the logit $l$ to efficiently compute a vector of counterfactual function components $\vec{f}_1$ and $\vec{f}_0$ where each element corresponds to a counterfactual function component obtained if all else was the same but a given parent's value was fixed to 1 or 0. Here $\odot$ represents the elementwise product. Line 8 passes the necessary information to the unit's children.}
\end{figure}

\begin{figure}
\begin{algorithm}[H]
\begin{algorithmic}[1]
\STATE Receive $\vec{x}$ from parents\\
\STATE $l=\vec{\theta}\cdot\vec{x}+b$\\
\STATE $p=\sigma(l)$\\
\STATE $\phi\sim \textit{Bernoulli}(p)$\\
\STATE Pass $\phi$ to children\\
\STATE Receive $\vec{q}_1,\vec{q}_0$ from child units\\
\STATE Receive $\vec{f}_0^d,\vec{f}_1^d$ from child function components with only direct connections\\
\STATE Receive $\vec{f}_0^c,\vec{f}_1^c$ from child function components which are also connected through children\\
\STATE Receive $G$, sum of downstream non-child function components
\STATE $f_0^c=\sum_i\vec{f}_0^c[i]$
\STATE $f_1^c=\sum_i\vec{f}_1^c[i]$
\STATE $f_0^d=\sum_i\vec{f}_0^d[i]$
\STATE $f_1^d=\sum_i\vec{f}_1^d[i]$
\STATE $q_1=\prod_{i}\vec{q}_1[i]$\\
\STATE $q_0=\prod_{i}\vec{q}_0[i]$\\
\STATE $\bar{q}=pq_1+(1-p)q_o$\\
\STATE $\vec{l}_1=l+\vec{\theta}\odot(1-\vec{x})$\\
\STATE $\vec{l}_0=l-\vec{\theta}\odot\vec{x}$\\
\STATE $\vec{p}_1=(1-\phi)(1-\sigma(\vec{l}_1))+\phi\sigma(\vec{l}_1)$\\
\STATE $\vec{p}_0=(1-\phi)(1-\sigma(\vec{l}_0))+\phi\sigma(\vec{l}_0)$\\
\STATE Pass $\vec{p}_1,\vec{p}_0$ to parents\\
\STATE $\vec{\theta}=\vec{\theta}+\alpha\sigma^{\prime}(l)\vec{x}\left(\frac{q_1f_1^c-q_0f_0^c}{\bar{q}}+\frac{q_1-q_0}{\bar{q}}G+f_1^d-f_0^d\right)$\\
\STATE $b=b+\alpha\sigma^{\prime}(l)\left(\frac{q_1f_1^c-q_0f_0^c}{\bar{q}}+\frac{q_1-q_0}{\bar{q}}G+f_1^d-f_0^d\right)$
\end{algorithmic}
\captionsetup{labelformat=empty}
\caption{$f$-HNCA algorithm for Bernoulli unit}\label{generalized_bernoulli_alg}
\end{algorithm}
\addtocounter{algorithm}{-1}
\captionof{algorithm}{Efficient implementation of $f$-HNCA for a Bernoulli unit, where function components are credited through all children, or none. We omit any direct dependence of function components on network parameters for conciseness. Lines 1-5 implement the forward pass, which takes input from the parents, uses it to compute the fire probability $p$ and samples $\phi\in\{0,1\}$. In the backward pass, the unit receives two vectors of probabilities $\vec{q}_1$ and $\vec{q}_0$, each with one element for each child unit of the current unit, as in Algorithm~\ref{bernoulli_alg}. The unit also receives vectors $\vec{f}^d_1$ and $\vec{f}^d_0$ containing counterfactual function components from function components with only direct connections. Likewise, $\vec{f}^c_1$ and $\vec{f}^c_0$ contain counterfactual function components from child function components with direct connections as well as additional connections mediated through children. Finally, G contains the cumulative sum of all function components which are downstream of the current unit but not directly connected. Lines 10-13 sum up the counterfactual function components. Lines 14 and 15 take the product of all child unit probabilities to compute $\prod_{C\in\ch(\Phi)}\pi_C(C|\pa(C)\setminus\Phi,\Phi=0/1)$. Line 16 computes the associated normalizing factor. Lines 17-20 use the already computed logit $l$ to efficiently compute a vector of probabilities $\vec{p}_1$ and $\vec{p}_0$ where each element corresponds to a counterfactual probability of $\phi$ if all else was the same but a given parent's value was fixed to 1 or 0. Here $\odot$ represents the elementwise product. Line 21 passes the necessary information to the unit's children. Lines 22 and 23 finally update the parameter using $\hat{G}^{\text{$f$-HNCA}}_{\Phi}$ with learning-rate hyperparameter $\alpha$.}
\end{figure}

In addition to the efficiency of computing counterfactual probabilities, for $f$-HNCA, we have to consider the efficiency of computing counterfactual function components $f^i_\Phi(\phi)$. For function components with no direct connection to a unit $\Phi$, this is trivial as $f^i_\Phi(\phi)=f^i$. If $f^i$ is directly connected, then implementing $f$-HNCA with efficiency similar to HNCA will require that we are able to compute $f^i_\Phi(\phi)$ from $f^i$ in constant time. This is the case, for example, if $f^i$ is a linear function followed by some activation. For example functions of the form $f^i=\log(\sigma(\vec{\theta}\cdot\vec{x}+b))$ which appear in the ELBO function used in our VAE experiments. Algorithm~\ref{reward_var_alg} presents pseudocode for efficiently computing counterfactual values for such function components, and passing them to connected units.

If only a subset of $\ch(\Phi)$ lies in $\widetilde{\an}(f^i)$ we could use $\ch^i(\Phi)=(\ch(\Phi)\cap\widetilde{\an}(f^i))$, or any superset, in place of $\ch(\Phi)$ in $\rho_\Phi(\phi)$. In this case, we would also have to consider the complexity of computing the HNCA estimator for each such $\ch^i(\Phi)$. In the worst case $\ch^i(\Phi)$ may be different for each $i$, meaning that $\rho_\Phi(\phi)$ may have to be separately computed for each $i$, requiring a product of up to $|\ch(\Phi)|$ numbers for each function component $f^i$. We leave open the question of how efficiently this can be done in general. For now, we focus on the case where either $\ch^i(\Phi)=\emptyset$ or $\ch^i(\Phi)=\ch(\phi)$. Focusing on this case allows us to rewrite the $f$-HNCA gradient estimator as follows:
\begin{equation*}
\begin{multlined}[t]
    \hat{G}_{\Phi}^{\text{$f$-HNCA}}=\sum_{\phi}\frac{\partial\pi_\Phi(\phi|\pa(\Phi))}{\partial\theta_\Phi}\Biggl(\rho_\Phi(\phi)\Biggl(\sum_{i:\ch^i(\Phi)\neq\emptyset,\Phi\in\widetilde{\pa}(f^i)}f^i_\Phi(\phi)+\\\sum_{i:\ch^i(\Phi)\neq\emptyset,\Phi\not\in\widetilde{\pa}(f^i)}f^i\Biggr)+\sum_{i:\ch^i(\Phi)=\emptyset,\Phi\in\widetilde{\pa}(f^i)}f^i_\Phi(\phi)\Biggr)+\sum_i\frac{\partial f^i}{\partial\theta_\Phi},
\end{multlined}
\end{equation*}
where $\rho_\Phi(\phi)=\frac{\prod\limits_{C\in\ch(\Phi)}\pi_C(C|\pa(C)\setminus\Phi,\Phi=\phi)}{\sum\limits_{\phi'}\pi_\Phi(\phi'|\pa(\Phi))\prod\limits_{C\in\ch(\Phi)}\pi_C(C|\pa(C)\setminus\Phi,\Phi=\phi')}$. Notice that we do not need to compute a different value of $\rho_\Phi(\phi)$ for each $f^i$, as we treat the dependence on children as either all or none. The three sums over function components from first to last handle: function components with both mediated and direct connection to $\Phi$, function components with only mediated connections to $\Phi$, and function components with only direct connections to $\Phi$.

Furthermore, if during the backward pass there are function components which we know have no direct connection to units further upstream, we can accumulate these in a sum and credit upstream units with the sum rather than separately computing the sum in each unit. This is analogous to accumulating the future return in reinforcement learning. 

Algorithm~\ref{generalized_bernoulli_alg} presents pseudocode for an efficient implementation of $f$-HNCA for a Bernoulli unit within a feedforward architecture where each function component is credited as being either downstream of every unit in the following layer or none.
\section{The Components of the $f$-HNCA Gradient Estimator have Lower Variance than the Associated Components of the REINFORCE Gradient Estimator}\label{Generalized_HNCA_low_variance}

Here, we verify that the components of the $f$-HNCA estimator with $\hat{G}^{\text{$f$-HNCA},i}_\Phi(\phi)$ have lower variance than the associated components of the analogous REINFORCE estimator. This is formalized in the following theorem:
\begin{theorem}
Let 
$$\hat{G}^{\text{$f$-HNCA},i}_\Phi(\phi)\defeq\begin{cases}
\sum_\phi\rho_\Phi(\phi)\frac{\partial\pi_\Phi(\phi|\pa(\Phi))}{\partial\theta_\Phi}f^i_\Phi(\phi) &\text{ if }\ch(\Phi)\cap\widetilde{\an}(f^i)\neq \emptyset\\
\sum_\phi\frac{\partial\pi_\Phi(\phi|\pa(\Phi))}{\partial\theta_\Phi}f^i_\Phi(\phi)&\text{ if }\ch(\Phi)\cap\widetilde{\an}(f^i)= \emptyset
\end{cases}
$$ 
where $\rho_\Phi(\phi)=\frac{\prod_{C\in\ch(\Phi)}\pi_C(C|\pa(C)\setminus\Phi,\Phi=\phi)}{\sum_{\phi'}\pi_\Phi(\phi'|\pa(\Phi))\prod_{C\in\ch(\Phi)}\pi_C(C|\pa(C)\setminus\Phi,\Phi=\phi')}$. Let $$\hat{G}_{\Phi}^{\text{RE},i}=\frac{\partial\log(\pi_\Phi(\Phi|\pa(\Phi)))}{\partial\theta_\Phi}f^i,$$
that is, the obvious generalization of REINFORCE to a specific function component. Then 
$$\Var(\hat{G}_{\Phi}^{\text{$f$-HNCA},i})\leq \Var(\hat{G}^{\text{RE},i}_{\Phi}).$$
\end{theorem}
\begin{proof}
We will separately consider the case where $\ch(\Phi)\cap\widetilde{\an}(f^i)\neq\emptyset$ and $\ch(\Phi)\cap\widetilde{\an}(f^i)=\emptyset$. First, when $\ch(\Phi)\cap\widetilde{\an}(f^i)\neq\emptyset$ 
We know from Equation~\ref{GHNCA_derivation_1} that can write $\hat{G}^{\text{$f$-HNCA},i}_{\Phi}$ as follows:
\begin{align*}
    \hat{G}^{\text{$f$-HNCA},i}_{\Phi}&=\E\left[\frac{\partial\log(\pi_\Phi(\Phi|\pa(\Phi)))}{\partial\theta_\Phi}f^i\middle|\mb(\Phi),f^i_\Phi(\phi)\right]\\
    &=\E\left[\hat{G}^{\text{RE},i}_{\Phi}\middle|\mb(\Phi),\widetilde{\pa}(f^i)\setminus\Phi\right].
\end{align*}
Now apply the law of total variance to rewrite the variance of the REINFORCE estimator as follows:
\begin{align*}
    \Var(\hat{G}^{\text{RE},i}_{\Phi})&=
    \E\left[\Var\left(\hat{G}^{\text{RE},i}_{\Phi}\middle|\mb(\Phi),\widetilde{\pa}(f^i)\setminus\Phi\right)\right]+\Var\left(\E\left[\hat{G}^{\text{RE},i}_{\Phi}\middle|\mb(\Phi),\widetilde{\pa}(f^i)\setminus\Phi\right]\right)\\
    &\geq \Var\left(\E\left[\hat{G}^{\text{RE},i}_{\Phi}\middle|\mb(\Phi),\widetilde{\pa}(f^i)\setminus\Phi\right]\right)\\
    &=\Var(\hat{G}^{\text{$f$-HNCA}}(\Phi)).
\end{align*}
For the case where $\ch^i(\Phi)=\emptyset$, we know from Equation~\ref{GHNCA_derivation_2} that
\begin{align*}
    \hat{G}^{\text{$f$-HNCA},i}_{\Phi}&=\E\left[\frac{\partial\log(\pi_\Phi(\Phi|\pa(\Phi)))}{\partial\theta_\Phi}f^i\middle|\pa(\Phi),\widetilde{\pa}(f^i)\setminus\Phi\right]\\
    &=\E\left[\hat{G}_{\Phi}^{\text{RE},i}\middle|\pa(\Phi),\widetilde{\pa}(f^i)\setminus\Phi\right].
\end{align*}
Now, again, apply the law of total variance to rewrite the variance in the REINFORCE estimator:
\begin{align*}
    \Var(\hat{G}^{\text{RE},i}_{\Phi})&=\E\left[\Var\left(\hat{G}^{\text{RE},i}_{\Phi}\middle|\pa(\Phi),\widetilde{\pa}(f^i)\setminus\Phi\right)\right]+\Var\left(\E\left[\hat{G}^{\text{RE},i}_{\Phi}\middle|\pa(\Phi),\widetilde{\pa}(f^i)\setminus\Phi\right]\right)\\
    &\geq \Var\left(\E\left[\hat{G}^{\text{RE},i}_{\Phi}\middle|\pa(\Phi),\widetilde{\pa}(f^i)\setminus\Phi\right]\right)\\
    &=\Var(\hat{G}^{\text{$f$-HNCA,i}}(\Phi)).
\end{align*}
\end{proof}

\section{Further Details of Discrete VAE Experiments}~\label{VAE_exp_details}
Here, we provide some additional detail on the methods used in our discrete VAE experiments.

We compare $f$-HNCA with REINFORCE and two stronger, unbiased, baselines for optimizing an ELBO of a VAE trained to generate MNIST digits. The other baselines are DisARM~\citep{dong2020disarm}, and REINFORCE leave one out (REINFORCE LOO; \citet{kool2019buy}).

REINFORCE LOO, based on the version used by~\citet{dong2020disarm}, samples two partial forward passes starting at each layer to compute its baseline. In other words, we first run a single forward pass to generate one sample form each $\vec{\Phi}_i=\vec{\phi}_i(1)$. All the function components that lie downstream of $\vec{\Phi}_i$ are summed up to produce one sample of the forward function components $\tilde{f}_i(1)$. This serves as the first of 2 samples used to construct the REINFORCE LOO gradient estimator in each layer. Then, in each layer, $i$ we also draw a second sample $\vec{\Phi}_i=\vec{\phi}_i(2)$ conditioned on $\vec{\phi}_{i-1}(1)$ (or $\vec{X}$ for $i=1$) all $\vec{\Phi}_j$ for $j>i$ are then resampled sequentially and the new sampled values used as input to the forward function components. This produces, for each layer, another sample of the forward function components which we'll call $\tilde{f}_i(2)$. This results in the following gradient estimator:
\begin{equation}\label{RLOO_grad}
    \hat{G}^{\text{RLOO}}(\Phi)=\begin{multlined}[t]\frac{1}{2}\Biggl(\frac{\partial\log(\pi_\Phi(\phi(1)|\pa(\Phi)))}{\partial\theta_\Phi}(\tilde{f}(1)-\tilde{f}(2))+\\
    \frac{\partial\log(\pi_\Phi(\phi(2)|\pa(\Phi)))}{\partial\theta_\Phi}(\tilde{f}(2)-\tilde{f}(1))\Biggr),
    \end{multlined}
\end{equation}
where we have suppressed the specific layer and written the estimator for a specific unit $\Phi$ in the vector $\vec{\Phi}_i$. Note that the computational cost of this procedure is quadratic in the number of layers, as we need to resample a partial forward pass to generate $\tilde{f}_i(2)$ for each layer $i$. DisARM has a similar computational requirement, requiring forward resampling to generate an antithetic sample in each layer.

We also experimented with another version of REINFORCE LOO that avoided this quadratic scaling of computational cost with number of layers. This second version of REINFORCE LOO used 2 independent forward passes for each input to construct a baseline, we call this REINFORCE LOO IS, for independent sample. Since REINFORCE LOO IS doesn't require sampling partial forward passes for each layer, it avoids a quadratic scaling of compute time with number of network layers which occurs for both DisARM and REINFORCE LOO. More precisely, rather than resampling in each layer, REINFORCE LOO IS simply generates 2 full forward passes, using the downstream function components of the first sample in each layer $i$ to define $\tilde{f}_i(1)$ and $\pi_\Phi(\phi(1)|\pa(\Phi)))$ and the downstream components of the second to define $\tilde{f}_i(2)$ and $\pi_\Phi(\phi(2)|\pa(\Phi)))$. The form of the resulting estimator is otherwise the same as Equation~\ref{RLOO_grad}. The drawback is that the baselines used for REINFORCE LOO IS will be less correlated, since unlike REINFORCE LOO its baseline uses a different sample for nodes upstream of the layer for which the baseline is being computed. Empirically we found this version to perform just slightly worse than the first version, hence we chose to omit the results to avoid clutter.

In $f$-HNCA with Baseline, for each layer, we maintain a scalar moving average of the sum of those components of $f$ with mediated connections (those highlighted in pink and orange in Figure~\ref{fig:generative_model}) and subtract it from the leftmost sum over $i$ in Equation~\ref{generalized_HNCA} to produce a centered learning signal. We use a discount rate $0.99$ for the moving average. For REINFORCE with baseline we use a similar moving average baseline, but in this case constructed as the sum of all downstream function components.


As in our contextual bandit experiments, we use dynamic binarization. Following~\citet{dong2020disarm}, our decoder and encoder each consist of a fully connected, stochastic feedforward neural network with 1, 2 or 3 layers, each hidden layer has 200 Bernoulli units. As in Section~\ref{bandit_experiments}, we train using ADAM optimizer with a learning rate $10^{-4}$ and batch-size of $50$. We train for $840$ epochs, approximately equivalent to the $10^6$ updates used by~\citet{dong2020disarm}. For consistency with prior work, we use Bernoulli units with a zero-one output. Unlike~\citet{dong2020disarm} we use ADAM to train the parameters of the prior as well, rather than using SGD.

For all methods, we train each unit based only on downstream function components as opposed to using the full function $f$. Also , for all methods, we train direct gradients (i.e. the right expectation in Equation~\ref{gradient_decomp}) with only a single sample per training example. In practice, it may be natural to use multiple samples in methods like REINFORCE LOO given that we draw multiple samples to construct the estimator of the left expectation anyways. This choice was made to reduce confounding, given we are mainly interested in how well different method estimate the left expectation. 

\section{Multisample Test-set Bounds}\label{test_set_ELBOs}
In this section, we report 100 sample ELBOs on the MNIST test-set for networks trained with each of the algorithms evaluated in Section~\ref{VAE_experiments}. Multi-sample bounds, as introduced by~\citet{burda2015importance} provide a tighter bound on the data likelihood under the generative model. Note that these results simply compute a multi-sample bound using the final trained encoder and decoder and, unlike~\citet{burda2015importance}, still use the single-sample ELBO as a training objective. These results are presented in Table~\ref{test_set_ELBOs_table}. These results show the same trend as the training ELBOs in Figure~\ref{fig:generative_model_plots}.
\begin{table}[htb]
\centering
\begin{tabular}{lllll}
                        & \textbf{1 Layer} & \textbf{2 Layer} & \textbf{3 Layer} &  \\ \cline{1-4}
HNCA                    &\textbf{-107.5$\pm$0.1}&-103.7$\pm$0.1&-102.1$\pm$0.2&  \\
HNCA with Baseline      &\ NA &\textbf{-97.3$\pm$0.1}& \textbf{-94.6$\pm$0.2}&  \\
DisARM                  &-108.2$\pm$0.2&-99.27$\pm$0.06&-96.7$\pm$0.1&  \\
REINFORCE LOO           &-108.3$\pm$0.1&-99.5$\pm$0.1&-96.9$\pm$0.1&  \\
REINFORCE               &-120.1$\pm$0.2&-115.1$\pm$0.1&-114.7$\pm$0.1&  \\
REINFORCE with Baseline &-110.6$\pm$0.1&-102.8$\pm$0.2&-100.2$\pm$0.1&  \\
\end{tabular}
\caption{100 sample test-set likelihood bounds for networks trained with each of the algorithms evaluated in Section~\ref{VAE_experiments}. Each cell provides the mean and 95\% confidence interval from 5 random seeds. The best result for each Layer count is written in bold.}\label{test_set_ELBOs_table}
\end{table}

\section{HNCA Ablation Results}\label{ablation}
\begin{figure}[htb]
    \includegraphics[width=\textwidth]{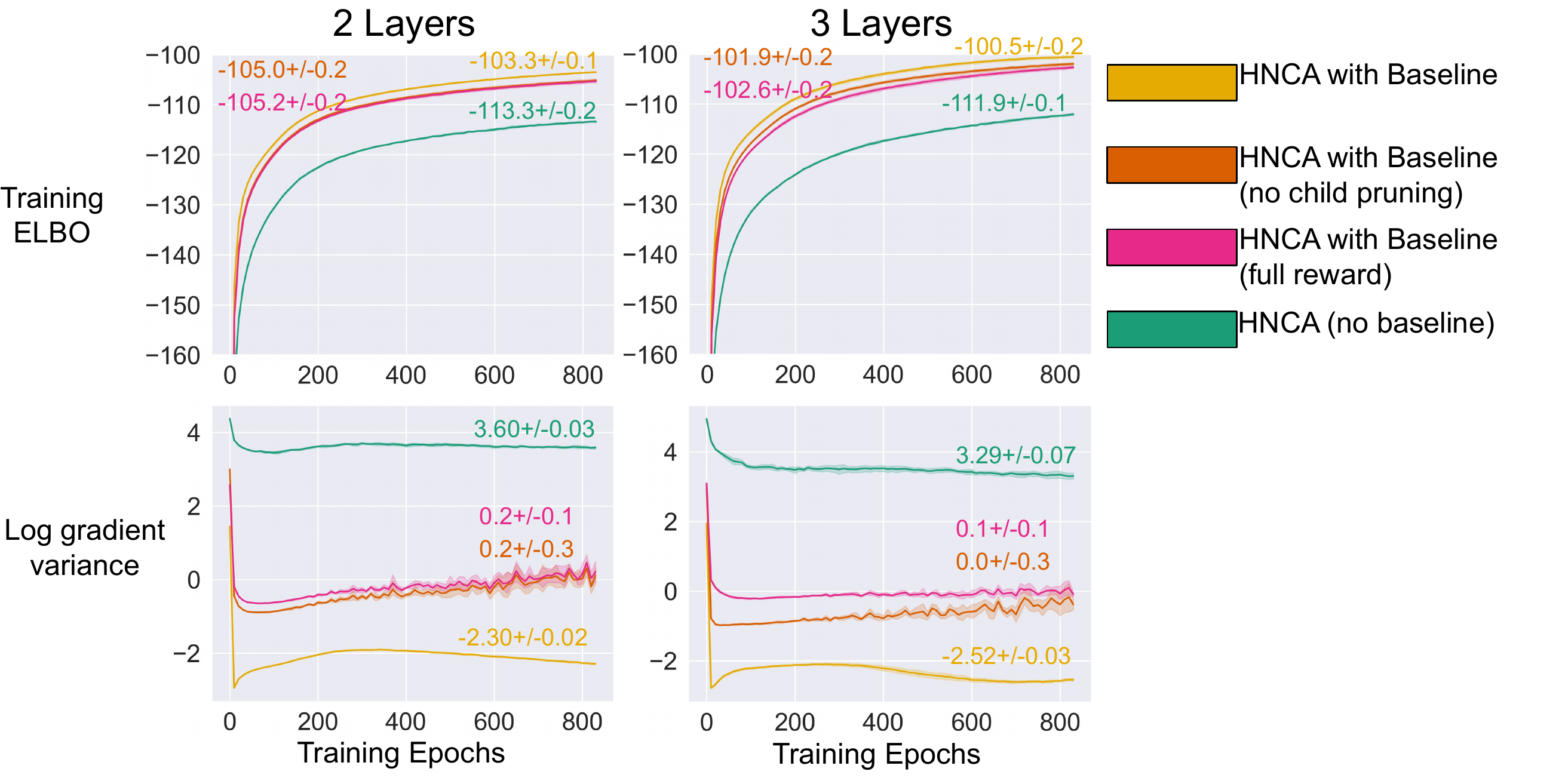}
    \caption{Training stochastic VAEs to generate MNIST digits with $f$-HNCA with Baseline with different aspects ablated. We omit the single-layer VAE as the ablations are not meaningful in this case. No child pruning refers to unnecessarily multiplying by $\rho_{\Phi}(\phi)$ even when no children have downstream connections to a function component, that is Equation~\ref{no_pruning_HNCA}. Full reward, does the same as no child pruning, in addition to unnecessarily including upstream function components in the estimator. For full reward, these additional function components are also included in the moving average baseline. Each line represents the average of 5 random seeds with error bars showing $95\%$ confidence interval. Final values at the end of training are written near each line in matching color. The top row shows the online training ELBO. The bottom row shows the natural logarithm of the mean gradient variance. Mean gradient variance is computed as the mean over parameters and batches of the per-parameter empirical variance over examples in a training batch of $50$. It appears that unnecessarily including children has a significant negative impact on $f$-HNCA with baseline, while the impact of including upstream function components is negligible.}
    \label{fig:HNCA_ablation_plots}
\end{figure}
In this section, we assess the impact of avoiding propagating credit through children in $f$-HNCA when a particular function component has only direct connections (those highlighted in green in Figure~\ref{fig:generative_model}). In particular, instead of using
\begin{equation}\label{original_HNCA}
\hat{G}^{\text{$f$-HNCA},i}_\Phi(\phi)\defeq\begin{cases}
\sum_\phi\rho_\Phi(\phi)\frac{\partial\pi_\Phi(\Phi|\pa(\Phi))}{\partial\theta_\Phi}f^i_\Phi(\phi) &\text{ if }\ch(\Phi)\cap\widetilde{\an}(f^i)\neq \emptyset\\
\sum_\phi\frac{\partial\pi_\Phi(\Phi|\pa(\Phi))}{\partial\theta_\Phi}f^i_\Phi(\phi)&\text{ if }\ch(\Phi)\cap\widetilde{\an}(f^i)= \emptyset
\end{cases},
\end{equation}
we simply use
\begin{equation}\label{no_pruning_HNCA}
\hat{G}^{\text{$f$-HNCA},i}_\Phi(\phi)\defeq\begin{cases}
\sum_\phi\rho_\Phi(\phi)\frac{\partial\pi_\Phi(\Phi|\pa(\Phi))}{\partial\theta_\Phi}f^i_\Phi(\phi) &\text{ if }\ch(\Phi)\neq \emptyset\\
\sum_\phi\frac{\partial\pi_\Phi(\Phi|\pa(\Phi))}{\partial\theta_\Phi}f^i_\Phi(\phi)&\text{ if }\ch(\Phi)= \emptyset
\end{cases},
\end{equation}
multiplying by $\rho_\Phi(\phi)$ as long as the unit has children, even if no children have downstream connections to the function component, that is even if $\ch(\Phi)\cap\widetilde{\an}(f^i)= \emptyset$. In this case we also include these function components in the subtracted baseline. We additionally investigate the impact of including redundant upstream function components in the HNCA gradient estimator. The results for the hierarchical VAE task are shown in Figure~\ref{fig:HNCA_ablation_plots}. Propagating credit through all children resulted in significantly worse performance for $f$-HNCA with Baseline. The additional impact of including upstream function components is minimal. Presumably, the subtracted baseline is able to mitigate the majority of increased variance resulting from including these function components.

\end{document}